\documentclass[journal]{IEEEtran}

\usepackage{cite}
\usepackage{amsmath}
\usepackage{amsfonts}
\usepackage{commath}
\usepackage{hyperref}
\usepackage{graphicx}
\usepackage{subcaption}
\usepackage{multirow}
\usepackage[table]{xcolor}
\usepackage[algo2e,ruled]{algorithm2e}
\usepackage{placeins}
\usepackage{bm}
\usepackage{balance}

\usepackage{amsthm}
\newtheorem{theorem}{\hspace{0pt}\bf Theorem}
\theoremstyle{theorem}

\newtheorem{definition}{\bf Definition}
\theoremstyle{definition}

\newtheorem{lemma}{Lemma}
\theoremstyle{lemma}

\newtheorem{assumption}{Assumption}
\theoremstyle{assumption}

\newtheorem{example}{\bf Example}
\theoremstyle{example}

\DeclareMathOperator*{\argmax}{arg\,max}
\DeclareMathOperator*{\argmin}{arg\,min}

\begin{document}
\title{Sufficiently Accurate Model Learning for Planning}

\author{Clark~Zhang,
        Santiago~Paternain,
        and~Alejandro~Ribeiro% <-this % stops a space
\thanks{C. Zhang and A. Ribeiro are with the Electrical Engineering department at the University of Pennsylvania: \{clarkz, aribeiro\}@seas.upenn.edu. S. Paternain is with the Electrical Computer and Systems Engineering department at the Rensselaer Polytechnic Institute: paters@rpi.edu }}% <-this % stops a space

\markboth{}%
{}

\maketitle

\begin{abstract}
Data driven models of dynamical systems help planners and controllers to provide more precise and accurate motions. Most model learning algorithms will try to minimize a loss function between the observed data and the model's predictions. This can be improved using prior knowledge about the task at hand, which can be encoded in the form of constraints. This turns the unconstrained model learning problem into a constrained one. These constraints allow models with finite capacity to focus their expressive power on important aspects of the system. This can lead to models that are better suited for certain tasks. This paper introduces the constrained \textit{Sufficiently Accurate} model learning approach, provides examples of such problems, and presents a theorem on how close some approximate solutions can be. The approximate solution quality will depend on the function parameterization, loss and constraint function smoothness, and the number of samples in model learning.
\end{abstract}

\section{Introduction}
Dynamics models play an essential role in many modern controllers and planners. This is for instance the case in Model Predictive Control (MPC) \cite{morari1999model}. They can also be used to compute feedforward terms for feedback controllers \cite{aastrom2010feedback}, or be used with a planner to find a trajectory from which stabilizing controllers can be generated \cite{hoffmann2008quadrotor, mellinger2011minimumsnap}. While model-free Reinforcement Learning techniques can solve similar problems without the use of a dynamics model \cite{schulman2015trust, khan2018learning, gu2017deep}, they generally do not scale to new problems or changes in problem parameters.
Model learning methods perform admirably when the models can approximate the dynamics of the system accurately. However, the performance of these controllers can be degraded with uncertainty in the model. While robust controllers can be designed to attempt to alleviate these issues \cite{petersen2012robust}, these methods typically perform conservatively due to having to cater to worst case approximations of the model. In addition, there may be effects that a robust controller designer may not be aware. For example, consider the problem of landing a quadrotor precisely at a target as shown in Figure \ref{fig:quad3d_traj}. There are complex aerodynamic effects associated when nearby surfaces cause disturbances to the airflow. This may result in large torques when the quadrotor is hovering close to the ground and hamper precise landings. This is known as the ``ground effect." These aerodynamic effects can be hard to model from just prior knowledge and may show up as a highly correlated, state dependent, non-zero mean noise. A common method that has been suggested to model similar effects is to learn or adjust a dynamics model with real data taken from running the system. 

For linear systems, this method of System Identification (SI) or Model Learning has many results \cite{qin2006overview, viberg1995subspace, pillonetto2010new, bruls1999linear, ljung1999system, pasquier2015robust} \textemdash such as recoverability of linear system dynamics when the data contains sufficient excitation. System identification methods have been proposed for non-linear systems \cite{ozer2011identification, billings1980identification, schoukens2017identification}. However, they suffer from issues such as the need for a large amount of data or the requirement of a special problem structure such as block-based systems.

While generic methods for non-linear systems do not provide the same theoretical guarantees, there has been some success in practice. In robotics, Gaussian Processes, Gaussian Mixture Models, or Neural Networks have been used to learn models of dynamics \cite{deisenroth2011pilco, nguyen2009local, levine2013guided, nagabandi2018neural}. A typical process for learning these models involves selecting a parameterized model, such as a neural network with a fixed number of layers and neurons, and choosing a loss function that penalizes the output of the model for not matching the data gathered from running the real system. Then, one optimizes the parameters by minimizing the empirical risk using, for instance, stochastic gradient descent like algorithms. This formulation assumes that all transitions are equally important, since it penalizes the mismatch between model and data uniformly on all portions of the state-action space. While this formulation has shown success in some applications, prior knowledge about the task and system can inform better learning objectives. A control designer may know that a certain part of the state space requires a certain accuracy for a robust controller to work well, or that some part of the state space is more important and should have hard constraints on the model accuracy. For example, to precisely land a quadrotor, a designer may note that the accuracy of modeling the complex ground effect forces is more important near the landing site. Incorporating this prior knowledge can lead to better performing models. 

To address the problem of incorporating prior knowledge into model learning, we introduced the idea of Sufficiently Accurate Model Learning in \cite{zhang2019learning}. This formulation is based on the inclusion of constraints in the optimization problem whose role is to introduce prior-knowledge about the importance of different state-control subsets. In the example of the quadrotor, notice that when the quadrotor is away from the surfaces, the ground effect is minor and thus, it is important to focus the learned model's expressiveness in the region of the state-space that is most heavily affected. This can be easily captured by a constraint that the average error in the important state-space regions is smaller than a desired value. These constraints will allow models with finite expressiveness concentrate on modeling important aspects of a system. One point to note is that this constrained objective can be used orthogonally to many existing methods. For example, the constrained objective can replace the unconstrained objective in \cite{levine2013guided, nagabandi2018neural}, and all other aspects of the methods can remain the same. The data can be collected the same way. The idea of using extra sensors during training for \cite{levine2013guided} need not change. While not trivial, the idea of constraints can also be applied to Gaussian process models such as \cite{deisenroth2011pilco, nguyen2009local}.

In its most generic form, the problem of model learning is an infinite dimensional non-convex optimization problem that involves the computation of expectations with respect to an unknown distribution over the state-action space. In addition, the formulation proposed here introduces constraints which seems to make the learning process even more challenging. However, in this work we show that solving this problem accurately is not more challenging than solving an unconstrained parametric learning problem. To reach this conclusion we solve a relaxation of the problem of interest with three common modifications: (i) function parameterization, (ii) empirical approximation, and (iii) dual problem solving. Function parameterization turns the infinite dimensional problem into one over finite function parameters. Empirical approximation allows for efficient computation of approximate expectations, and solving the dual problem leads to a convex unconstrained optimization problem. The three approximations introduced however may not yield solutions that are good approximations of the original problem. To that end, we establish a bound on the difference of the value of these solutions. This gap between the original and approximate problem depends on the number of samples of data as well as the expressiveness of the function approximation (Theorem \ref{thm:empirical_duality_gap}). In particular, the bound can be made arbitrarily small with sufficient number of samples and with the selection of a rich enough function approximator. This implies that solving the functional constrained problem is nearly equivalent to solving a sequence of unconstrained approximate problems using primal-dual methods.

This paper extends \cite{zhang2019learning} to the case of empirical samples and presents Theorem \ref{thm:empirical_duality_gap} that relates number of samples, function approximation expressiveness, and loss function smoothness to approximation error. In addition, there is experimental validation of Theorem \ref{thm:empirical_duality_gap} as well as different examples to showcase the framework of Sufficiently Accurate model learning. This paper is organized as follows: Section \ref{sec:saml} introduces the \textit{Sufficiently Accurate} model learning framework,  Section \ref{sec:duality_gap} presents in detail the three approximations introduced to solve the problem as well as a result that bounds the error on the approximate problem (Theorem \ref{thm:empirical_duality_gap}). Section \ref{sec:duality_gap_proof} provides the proof for the main theorem. Section \ref{sec:primal_dual} presents a simple primal-dual method to solve the constrained problem, while experimental results are presented in Section \ref{sec:experiments} on a double integrator system with friction, a ball bouncing task, and a quadrotor landing with ground effect in simulation. Theoretical results are experimentally tested for the double integrator system. Section \ref{sec:conclusion} presents the paper's conclusions.

\section{Sufficiently Accurate Model Learning}\label{sec:saml}
In this paper we consider the problem of learning a discrete time dynamical system. Let us denote $t\in \mathbb{Z}$ as the time index and let $x_t\in\mathbb{R}^n$, $u_t\in\mathbb{R}^p$ be the state and input of the system at time $t$. The dynamical system of interest is represented by a function $f:\mathbb{R}^n\times\mathbb{R}^p\to \mathbb{R}^n$ that relates the state and input at time $t$ to the state at time $t+1$
\begin{equation}
    x_{t+1} = f(x_t, u_t)
\end{equation}
One approach to System Identification or Model Learning consists of fitting an estimated dynamical model, $\phi : \mathbb{R}^n \times \mathbb{R}^p \rightarrow \mathbb{R}^n$, to state transition data \cite{deisenroth2011pilco}. This state transition data consists of tuples of $s = (x_t, u_t, x_{t+1})$ drawn from a distribution $\mathcal{S}_\mathcal{D}$ with the sample space $\mathcal{S} \subseteq \mathbb{R}^n \times \mathbb{R}^p \times \mathbb{R}^n$. The estimated model $\phi$ belongs to a class of functions, $\Phi$, of which $f$ is an element. $\Phi$ could be, for example, the space of all continuous functions. Then, the problem of model learning reduces to finding the function $\phi$ in the class $\Phi$ that best fits the transition data. The figure of merit is a loss function $\ell:\mathcal{S}\times \Phi \to \mathbb{R}$. With these definitions, the problem of interest can be written as the following stochastic optimization problem
\begin{equation} \label{eq:unconstrained_problem}
\min_{\phi \in \Phi} \mathbb{E}_{s \sim \mathcal{S}_\mathcal{D}} [ l(s, \phi) ].
\end{equation}
The loss function needs to be selected to encourage the model $\phi$ to match its output with the state transition data. An example of a loss function is the p-norm, $l(s, \phi) = \norm{\phi(x_t, u_t) - x_{t+1}}_p$. For $p=1$, this reduces to a sum of absolute differences between each output dimension of $\phi$ and the true next state, $x_{t+1}$. When $p=2$, this is simply a euclidean loss. Other common losses can include the Huber loss and, in discrete state settings, a 0-1 loss.

Often times, one can derive, from first principles, a model $\hat{f}$ of the dynamical system of interest. Depending on the complexity of the system of, these models may be inaccurate since they may ignore hard to model dynamics or higher order effects. For instance, one can derive a model $\hat{f}$ for a quadrotor from rigid body dynamics where forces and torques are functions of each motor's propeller speed. However, the accuracy of this model will depend on other effects that are harder to model such as aerodynamic forces near the ground. If these aerodynamic effects are ignored, it can result in a failure to control the system or in poor performance \cite{sanchez2017characterization}. In these cases, the target model, denoted by $\tilde{\phi}$, can decomposed as the sum of an analytic model and an error
\begin{equation}
    \tilde{\phi}(x_t, u_t) = \hat{f}(x_t, u_t) + \phi(x_t, u_t).
\end{equation}
The learning of the error term\textemdash or residual model\textemdash fits the framework described in \eqref{eq:unconstrained_problem}. For instance, for the $p$ norm loss can be modified to take the form $l(s, \phi) = \norm{ \tilde{\phi}(x_t, u_t) - x_{t+1} }_p = \norm{ \phi(x_t, u_t) - (x_{t+1} - \hat{f}(x_t, u_t)) }_p$. 

A characteristic of the classic model learning problem defined in \eqref{eq:unconstrained_problem} is that errors are uniformly weighted across the state-input space. In principle, one can craft the loss so to represent different properties on different subsets of the state-input space. However, this design is challenging, system dependent, and dependant on the transition data available. In contrast, our approach aims to exploit prior knowledge about the errors and how they impact the control performance. For instance, based on the analysis of robust controllers one can have bounds on the error required for successful control. This information can be used to formulate the \emph{Sufficiently Accurate} model learning problem, where we introduce the prior information in the form of constraints. Formally, we encode the prior information by introducing $K\in\mathbb{N}$ functions $g_k : \mathcal{S} \rightarrow \mathbb{R}$. 
Define in addition, a collection of subsets of transition tuples where this prior information is relevant, $\mathcal{S}_k\subset \mathcal{S}$ and corresponding indicator functions $\mathbb{I}_k(s) : \mathcal{S} \rightarrow \{0, 1\}$ taking the value one if $s\in\mathcal{S}_k$ and zero otherwise. With these definitions the sufficiently accurate model learning problem is defined as
%and a euclidean loss, errors are uniformly weighted in all parts of the state-control space. While the loss function can be formulated to weight specific errors more by adding multipliers to them, it can be hard to determine an appropriate multiplier. We want to be able to incorporate prior knowledge such as ``The errors in this part of the state-control space should be at least $\epsilon$ accurate." To this end, we formulate a ``Sufficiently Accurate" model learning objective
\begin{align} \label{eq:functional_primal}
    P^\star = &\min_{\phi \in \Phi} \mathbb{E}_{s \sim \mathcal{S}_\mathcal{D}}[l(s, \phi) \mathbb{I}_{0}(s)]\\
    & \text{s.t.  } \mathbb{E}_{s \sim \mathcal{S}_\mathcal{D}}[g_k(s, \phi) \mathbb{I}_k(s)] \leq 0, k=1, 2, \hdots, K  \nonumber
\end{align}
Note that the sets $\mathcal{S}_k$ that define the indicator variables $\mathbb{I}_k(s)$ are not necessarily disjoint. In fact, in practice, they are often not. The sets can be arbitrary and have no relation to each other. Examples of how these sets are used are given in examples at the end of this section.
Notice that the \eqref{eq:functional_primal} is an infinite dimensional problem since the optimization variable is a function and it involves the computation of expectations with respect to a possibly unknown distribution. An approximation to this problem is presented in Section \ref{sec:duality_gap}. For technical reasons, the functions $g_k$ and $l$ should be expectation-wise Lipschitz continuous.
\begin{assumption}\label{assumption_lipschitz}
	The functions $\ell_0(s, \phi) = l(s, \phi)\mathbb{I}_0(s)$ and $\boldsymbol{g}_k(s, \phi) = g_k(s, \phi)\mathbb{I}_k(s)$ are $L$-expectation-wise Lipschitz continuous in $\Phi$, i.e., 
	\begin{equation}
	\mathbb{E}_s \norm{ \ell_0(s, \phi_1) - \ell_0(s, \phi_2) }_\infty \leq L \mathbb{E}_s \norm{ \phi_1 - \phi_2 }_\infty, \forall \phi_1, \phi_2 \in \Phi
	\end{equation}
	Here, $\norm{ \phi_1 - \phi_2 }_\infty$ is the infinity norm for functions which is defined as
	$\norm{ f }_\infty = \sup_x\abs{f(x)}$.
\end{assumption}
The expectation-wise Lipschitz assumption is a weaker assumption than Lipschitz-continuity, as any Lipschitz-continuous function with a Lipschitz constant $L$ is also expectation-wise Lipschitz-continuous with a constant of $L$. In particular, the loss functions in Example \ref{example:selective_accuracy} and \ref{example:normalized_objective} are expectation-wise Lipschitz-continuous with some constant (cf., Appendix \ref{appendix:continuity}). There is no assumption that the functions should be convex or continuous.

Before we proceed, we present two examples of \textit{Sufficiently Accurate} model learning. For notational brevity, when an expectation does not have a subscript, it is always taken over $s \sim \mathcal{S}_\mathcal{D}$.
\begin{example} \label{example:selective_accuracy}
\textbf{Selective Accuracy}\\

\begin{equation}
\label{eq:selective_accuracy}
\begin{aligned}
    &\min_{\phi \in \Phi} \mathbb{E} \norm{ \phi(x_t, u_t) - x_{t+1} }_2  \\
    &\text{s.t.  } \mathbb{E} \norm{ \phi(x_t, u_t) - x_{t+1} }_2 \mathbb{I}_A(s) \leq \epsilon_c
\end{aligned}
\end{equation}
This problem is a simple modification of \eqref{eq:unconstrained_problem}. It has the same objective, but adds a constraint that a certain state-control subset, defined by a set $A$, should be within $\epsilon_c$ accuracy. The indicator variable $\mathbb{I}_A(s)$ will be 1 when $s$ is in the set $A$. Here, $g(s, \phi) = \norm{ \phi(x_t, u_t) - x_{t+1} }_2 - \frac{\epsilon_c}{\mathbb{E} \mathbb{I}_A(s)}$. This formulation allows you to trade off the accuracy in one part of the state-control space with everything else as it may be more important to a task. Another use case can be to provide an error bound for robust controllers. This is the formulation used in the quadrotor precise landing experiments detailed later in Section \ref{sec:quadrotor_experiments}, where the set $A$ is defined to be all states close to the ground where the ground effect is more prominent.  \ref{sec:quadrotor_experiments}.
\end{example}

\begin{example} \label{example:normalized_objective}
\textbf{Normalized Objective}\\
\begin{equation}
\label{eq:normalized_objective}
\begin{aligned}
    &\min_{\phi \in \Phi} \mathbb{E} \frac{\norm{ \phi(x_t, u_t) - x_{t+1} }_2}{ \norm{ x_{t+1} }_2} \mathbb{I}_{A}(s)\\
    & \text{s.t.  } \mathbb{E} \norm{ \phi(x_t, u_t) - x_{t+1} }_2 \mathbb{I}_{A^C}(s) \leq \epsilon_c
\end{aligned}
\end{equation}
where $\mathbb{I}_{A}$ is the indicator variable for the subset $A = \{s \in \mathcal{S} : \norm{ x_{t+1} }_2 \geq \delta_c \}$, and $A^C$ is the complement of the set $A$. This problem formulation looks at minimizing an objective such that the error term is normalized by the size of the next state. This can be useful in cases where the states can take values in a very large range. An error of 1 unit can be large if the true value is 0.1 units, but it is a small error if the true value is 100 units. The set $A$ contains all data samples where the true next state is large enough for this to be significant. This can reduce numerical issues when the denominator is small. For all small state values, the error is simply bounded by $\epsilon_c$. From a practical point of view, sensors will always have noise. When the state is small, the ``true" measurement of the state can be dominated by noise, and the model can be better off just bounding the error rather than focusing on fitting the noise. This is the formulation used in the ball bouncing experiment in Section \ref{sec:ball_bouncing_experiment}, where the we would like the errors in velocity prediction to be scaled to the speed, and all errors below a small speed can be constrained with a simple constant.
\end{example}

\section{Problem Approximation}\label{sec:duality_gap}
The unconstrained problem \eqref{eq:unconstrained_problem} and the constrained problem \eqref{eq:functional_primal} are functional optimization problems. In general, these are infinite dimensional and usually intractable. Instead of optimizing over the entire function space $\Phi$, one may look at function spaces, $\Phi_\theta \subset \Phi$, parameterized by a $d$-dimensional vector $\theta \in \Theta = \mathbb{R}^d$. Examples of these classes of functions are linear functions of the form $\phi_\theta(x, u) = \theta_x^\top x + \theta_u^\top u$ where $\theta = [\theta_x, \theta_u]$ is a vector of weights for the state and control input. More complex function approximators, such as neural networks, may be used to express a richer class of functions \cite{nagabandi2018neural, hornik1989multilayer}. Restricting the function space poses a problem in that the optimal solution to \eqref{eq:functional_primal} may no longer exist in the set $\Phi_\theta$. The goal under these circumstances should be to find the closest solution in $\Phi_\theta$ to the true optimal solution $\phi^\star$. Additionally, the expectations of the loss and constraint functions are in general intractable. The distributions can be unknown or hard to compute in closed form. In practice, the expectation is approximated with a finite number of data samples $s_i\sim\mathcal{S}_\mathcal{D}$ with $i=1,\ldots,N$. This yields the following empirical parameterized risk minimization problem
\begin{align}\label{eq:empirical_primal} 
    P_N^\star = &\min_{\theta \in \Theta} \frac{1}{N} \sum_{i=1}^N l(s_i, \phi_\theta) \mathbb{I}_0(s_i) \\
    &\text{s.t.  } \frac{1}{N} \sum_{i=1}^N g_k(s_i, \phi_\theta) \mathbb{I}_k(s_i) \leq 0, k=1, 2, \hdots, K \nonumber
\end{align}

While both function and empirical approximations are common ways to simplify the problem, the approximate problem introduced in \eqref{eq:empirical_primal} is still a constrained optimization problem and can be difficult to solve in general as it can be nonconvex in the parameters $\theta$. This is the case for instance when the function approximator is a neural network. One approach to solve this problem is to solve the dual problem associated with \eqref{eq:empirical_primal}. To aid in the definition of the dual problem, we first define the dual variables (also known as Lagrange multipliers), $\lambda \in \mathbb{R}_+^K$, along with the Lagrangian associated with \eqref{eq:empirical_primal}
\begin{equation}\label{eq:empirical_lagrangian}
    \mathcal{L}_N(\theta, \lambda) = \frac{1}{N} \sum_{i=1}^N \ell_0(s_i, \phi_\theta) + \lambda^\top \frac{1}{N}\sum_{i=1}^N \boldsymbol{g}(s_i, \phi_\theta).
\end{equation}
Here, the symbol, $\ell_0(s_i, \phi_\theta)$, is defined as $l(s_i, \phi_\theta) \mathbb{I}_0(s)$ to condense the notation. Similarly, the bolded vector, $\boldsymbol{g}(s_i, \phi_\theta)$ is a vector where the $k^{\text{th}}$ entry is defined as $g_k(s_i, \phi_\theta) \mathbb{I}_k(s)$. The dual problem is now defined as
\begin{equation} \label{eq:empirical_dual}
    D_N^\star = \max_{\lambda \geq 0} \min_{\theta \in \Theta} \mathcal{L}_N(\theta, \lambda)
\end{equation}
Notice that \eqref{eq:empirical_dual} is similar to a regularized form of \eqref{eq:empirical_primal} where each constraint is penalized by a coefficient $\omega_k$
\begin{equation}\label{eq:regularized}
D_N^\star = \min_{\theta \in \Theta} \frac{1}{N} \sum_{i=1}^N \ell_0(s_i, \phi_\theta) + \boldsymbol{\omega}^\top \frac{1}{N} \sum_{i=1}^N \boldsymbol{g}(s_i, \phi_\theta).
\end{equation}
Adding this type of regularization can weight certain state-action spaces more. In fact, if $\omega_k$ is chosen to be $\lambda_N^\star$, solving \eqref{eq:regularized} would be equivalent to solving \eqref{eq:empirical_dual}. However, arbitrary choices of $\omega_k$ provide no guarantees on the constraint function values. By defining the constraint functions directly, constraint function values are determined independent of any tuning factor. For problems where strong guarantees are required or easier to define, the Sufficiently Accurate framework will satisfy them by design. An alternative interpretation is that \eqref{eq:empirical_dual} provides a principled way of selecting the regularization coefficients. In Section \ref{sec:primal_dual}, we discuss an implementation of a primal dual algorithm to do so.

The dual problem has two important properties that hold regardless of the structure of the optimization problem \eqref{eq:empirical_primal}. For any $\theta$ that minimizes the Lagrangian $\mathcal{L}_N$, the resulting function\textemdash termed the dual function\textemdash is concave on the multipliers, since it is the point-wise minimum of linear functions (see e.g. \cite{boyd2004convex}). Therefore, its maximization is tractable and it can be done efficiently for instance using stochastic gradient descent. 
In addition, the dual function is always a lower bound on the value $P_N^\star$ and in that sense solving the dual problem \eqref{eq:empirical_dual} provides the tightest lower bound. In the case of convex problems (that fulfill Slater's Condition), it is well known that the problems have zero duality gap, and therefore $P_N^\star=D_N^\star$ \cite[Section 5.2.3]{boyd2004convex}. However, the problem \eqref{eq:empirical_primal} is non-convex and a priori we do not have guarantees on how far the values of the primal and the dual are. Moreover, recall that the primal problem in \eqref{eq:empirical_primal} is an empirical approximation of the problem that we are actually interested in solving \eqref{eq:functional_primal}.

The previous discussion leads to the question about the quality of the solution \eqref{eq:empirical_dual} as an approximation to \eqref{eq:functional_primal}. The duality gap is defined as the difference between the primal and dual solutions of the same problem. Here, the gap is the difference between the primal and the dual of different but closely related problems. Hence, the quantity we are interested in bounding is the 
%
%\subsection{The Surrogate Duality Gap}
%We have now approximated the problem we want to solve in \eqref{eq:functional_primal} with \eqref{eq:empirical_dual}. This now raises a question about the quality of the solution to \eqref{eq:empirical_dual}. Does solving the approximate problem yield a good solution to the original problem we want to solve? 
%One quantity of interest is now the 
\textit{surrogate} duality gap defined as
\begin{equation}\label{eq:surrogate_duality_gap}
    |P^\star - D_N^\star|.
\end{equation}
%This quantity represents the difference in the solution to \eqref{eq:functional_primal} and the solution to \eqref{eq:empirical_dual}. As a reminder, \eqref{eq:functional_primal} is the problem we would like to solve, and \eqref{eq:empirical_dual} is a problem we can feasibly try to solve. 

%We would like to find a bound on the surrogate duality gap, \eqref{eq:surrogate_duality_gap}, which is an indication of how good of an approximation \eqref{eq:empirical_dual} is to \eqref{eq:functional_primal}. This quantity is hard to bound in general, but with assumptions on the loss and function approximation, some guarantees exist. 
%
To provide specific bounds for the difference in the previous expression we consider the family of function classes $\Phi_\theta$ termed $\epsilon$-universal function approximators. We define this notion next. 
\begin{definition}\label{def:epsilon_unviersal}
The function class $\Phi_\theta$ is an $\epsilon$-universal function approximator for $\Phi$ if, for any $\phi \in \Phi$, there exists a $\phi_\theta \in \Phi_\theta$ such that $\mathbb{E}_{s \sim \mathcal{S}_\mathcal{D}}\norm{ \phi(s) - \phi_\theta(s) }_\infty \leq \epsilon$.
\end{definition}
To provide some intuition on the definition consider the case where $\Phi$ is the space of all continuous function, the above property is satisfied by some neural network architecture. That is, for any $\epsilon$, there exists a class of neural network architectures, $\Phi_\theta$ such that $\Phi_\theta$ is an $\epsilon$-universal approximator for the set of continuous functions \cite[Corollary 2.1]{hornik1989multilayer}. Thus, for any dynamical system with continuous dynamics, this assumption is mild. Other parameterizations, such as Reproducing Kernel Hilbert Spaces, are $\epsilon$-universal as well \cite{sriperumbudur2010relation}. Notice that the previous definition is an approximation on the total norm variation and hence it is a milder assumption than the universal approximation property that fully connected neural networks exhibit \cite{hornik1989multilayer}. 

Next, we define an intermediate problem on which the surrogate duality gap depends: a perturbed version of problem \eqref{eq:functional_primal} where the constraints are relaxed by $L \epsilon\geq 0$ where $L$ is the constant defined in Assumption \ref{assumption_lipschitz} and $\epsilon$ the universal approximation constant in Definition \ref{def:epsilon_unviersal}
\begin{equation}
\label{eq:perturbed_primal}
\begin{aligned}
P_{L\epsilon}^\star = &\min_{\phi \in \Phi} \mathbb{E}_{s \sim \mathcal{S}_\mathcal{D}}\lbrack \ell_0(s, \phi)\rbrack \\
&\text{s.t.  } \mathbb{E}_{s \sim \mathcal{S}_\mathcal{D}}\lbrack \boldsymbol{g}(s, \phi)\rbrack + \boldsymbol{1} L \epsilon \leq 0.
\end{aligned}
\end{equation}
$\boldsymbol{1}$ is a vector of ones. The perturbation results in a problem whose constraints are tighter as compared to \eqref{eq:functional_primal}. 
The set of feasible solutions for the perturbed problem \eqref{eq:perturbed_primal} is a subset of the feasible solutions for the unperturbed problem \eqref{eq:functional_primal}. The perturbed problem accounts for the approximation introduced by the parameterization. In the worst case scenario, if the problem \eqref{eq:perturbed_primal} is infeasible, the parameterized approximation of \eqref{eq:empirical_primal} may turn infeasible as the number of samples increases. 

Let $\lambda_{L\epsilon}^\star$ be the solution to the dual of \eqref{eq:perturbed_primal}
\begin{equation}
\label{eq:perturbed_dual}
    \lambda_{L\epsilon}^\star = \argmax_{\lambda \geq 0} \min_{\phi \in \Phi} \mathbb{E}\lbrack \ell_0(s, \phi)\rbrack + \lambda^\top (\mathbb{E}\lbrack \boldsymbol{g}(s, \phi)\rbrack + \boldsymbol{1} L \epsilon)
\end{equation}

With these definitions, we can present the main theorem that bounds the surrogate duality gap.
\begin{theorem}
Let $\Phi$ be a compact class of functions over a compact space such that there exists $\phi \in \Phi$ for which \eqref{eq:functional_primal} is feasible, and let $\Phi_\theta$ be an $\epsilon$-universal approximator of $\Phi$ as in Definition \ref{def:epsilon_unviersal}. Let the space of Lagrange multipliers, $\lambda$, be a compact set as in \cite{nedic2009subgradient}. In addition, let Assumption \ref{assumption_lipschitz} hold and let $\Phi_\theta$ satisfy the following property
            \begin{equation}
                \lim_{N\rightarrow \infty} \frac{H^{\Phi_\theta}({\delta} - \epsilon L (\norm{\lambda_{L\epsilon}^\star}_1 + 1), N)}{N} = 0
            \end{equation}
            where $\norm{\lambda_{L\epsilon}^\star}_1$ is the optimal dual variable for the problem \eqref{eq:perturbed_dual}, $L$ is the Lipschitz constant for the loss function, and $H^{\Phi_\theta}$ is the random VC-entropy \cite[section II.B]{vapnik1999overview}. Note that both arguments for $H^{\Phi_\theta}$ must be positive.
            Then $P^{\star}$ and $D^{\star}_N$, the values of \eqref{eq:functional_primal} and \eqref{eq:empirical_dual} respectively, satisfy
    \begin{equation}\label{eq:duality_gap_limit}
        \lim_{N\rightarrow \infty} \mathbb{P}\left(|P^\star - D_N^\star| \leq {\delta}\right) = 1,
    \end{equation}
    where the probability is over independent samples $\{s_1, s_2, \hdots, s_N\}$ drawn from the distribution $\mathcal{S}$ as defined in problem \eqref{eq:empirical_primal}.
    
\label{thm:empirical_duality_gap}
\end{theorem}

\begin{proof}
See Section \ref{sec:duality_gap_proof}
\end{proof}

The intuition behind the theorem is that given some acceptable surrogate duality gap, $\delta$, there exists a neural network architecture, $\Phi_\theta$, and a number of samples, $N$ such that the probability that the solution to \eqref{eq:empirical_dual} is within $\delta$ to the solution to \eqref{eq:functional_primal} is very high. 
The choice of neural network will influence the value of $\epsilon$ and $\lambda_{L\epsilon}^\star$. These in turn will decide the duality gap, $\delta$, as the quantity $\delta - \epsilon L(\norm{\lambda_{L\epsilon}^\star}_1 + 1)$ must be positive. 
A larger neural network will correspond to a smaller $\epsilon$ which will also has an impact on the perturbed problem \eqref{eq:perturbed_primal}. A smoother function and smaller $\epsilon$ will lead to smaller perturbations. Smaller perturbations can lead to a smaller dual variable, $\lambda_{L\epsilon}^\star$. Thus, larger neural networks and smoother dynamic systems will have smaller duality gaps. If $L\epsilon$ is large, then the perturbed problem may be infeasible. In theses cases, $\lambda_{L\epsilon}^\star$ will be infinite. This corresponds to problems where the function approximation simply can not satisfy the original constraint functions. For example, using constant functions to approximate a complicated system may violate the constraint functions $g_k$ for all possible constant functions. Thus, no $\delta$ exists to bound the solution as the parameterized empirical problem \eqref{eq:empirical_primal} has no feasible solution.  This theorem suggests that with a good enough function approximation and large enough $N$, solving \eqref{eq:empirical_dual} is a good approximation to solving \eqref{eq:functional_primal} with large probability.

There are some details to point out in Theorem \ref{thm:empirical_duality_gap}. First, the function $H^{\Phi_\theta}$ a complicated function that will usually scale with the size of the neural network. A larger neural network will lead to a smaller $\epsilon$, but may require a larger number of samples $N$ to adequately converge to an acceptable solution. The assumption on the limiting behavior of $H^{\Phi_\theta}$ is fufilled by some neural network architectures \cite{bartlett2017spectrally}, but the general behavior of this function for all neural network architectures is still a topic of research. Additionally, we assume the space of Lagrange multipliers is a compact set. This will imply, along with compact state-action space, that $\mathcal{L}$ is bounded. A finite Lagrange multiplier is a reasonable assumption as the problem \eqref{eq:functional_primal} is required to be feasible \cite{nedic2009subgradient}.

The bound established in Theorem \ref{thm:empirical_duality_gap} depends on quantities that are in general difficult to estimate, These include $H^{\Phi_\theta}$, $\norm{\lambda_{L\epsilon}^\star}_1$, $L$, $\epsilon$. Thus, while this theorem provides some insights on how these quantities influence the gap between solutions, it is mainly a statement of the existence of such values that can provide a desired result. In practice, this result can be achieved by choosing increasing the sizes of neural networks as well as data samples until the desired performance is reached. Note that the theorem follows our intuition that larger neural networks and more data will give us more accurate result. However, this theorem formalizes not only that it is more accurate, but that the error will tend to 0 as number of samples and number of parameters increase. 

\section{Proof of Theorem \ref{thm:empirical_duality_gap}}\label{sec:duality_gap_proof}
To begin, we define an intermediate problem
\begin{equation}
\label{eq:parameterized_primal}
\begin{aligned}
	P_\theta^\star = &\min_{\theta \in \Theta} \mathbb{E}_{s \sim \mathcal{S}_\mathcal{D}} [\ell_0(s, \phi_\theta)] \\
	&\textrm{s.t.  } \mathbb{E}_{s \sim \mathcal{S}_\mathcal{D}}[\boldsymbol{g}(s, \phi_\theta)] \leq 0.
\end{aligned}
\end{equation}
Note that this is the unperturbed version of \eqref{eq:perturbed_primal}. As a reminder, this problem uses a class of parameterized functions, but does not use data samples to approximate the expectation. Thus, it can be seen as a step in between \eqref{eq:functional_primal} and \eqref{eq:empirical_primal}. As with the dual problem to \eqref{eq:empirical_primal}, we can define the Lagrangian associated with \eqref{eq:parameterized_primal}
\begin{equation}
    \label{eq:parameterized_lagrangian}
    \mathcal{L}_\theta(\theta, \lambda) =  \mathbb{E}_{s \sim \mathcal{S}_\mathcal{D}}[\ell_0(s, \phi_\theta)] + \lambda^\top \mathbb{E}_{s \sim \mathcal{S}_\mathcal{D}}[\boldsymbol{g}(s, \phi_\theta)]
\end{equation}
and the dual problem
\begin{equation}
\label{eq:parameterized_dual}
D_\theta^\star = \max_{\lambda \geq 0} \min_{\theta \in \Theta} \mathcal{L}_\theta(\theta, \lambda).
\end{equation}

Using this intermediate problem, we can break the bound $\abs{P^\star - D_N^\star}$ into two components.
\begin{equation}
\label{eq:full_bound}
    \begin{aligned}
        \abs{P^\star - D_N^\star} &= \abs{(P^\star - D_\theta^\star) + (D_\theta^\star - D_N^\star)}\\
        &\leq \abs{P^\star - D_\theta^\star} + \abs{D_\theta^\star - D_N^\star}
    \end{aligned}
\end{equation}
As a reminder, $P^\star$ is the solution to the problem we want to solve in \eqref{eq:functional_primal}, $D_N^\star$ is the solution to the problem \eqref{eq:empirical_dual} we can feasibly solve, and $D_\theta^\star$ is the solution to an intermediate problem \eqref{eq:parameterized_dual}. The first half of this bound, $\abs{P^\star - D_\theta^\star}$, is the error that arises from using a parameterized function and dual approximation. The second half of this bound, $D_\theta^\star - D_N^\star$, is the error that arises from using empirical data samples. It can be seen as a kind of generalization error. The proof will now be split into two parts that will find a bound for each of these errors.

\subsection{Function Approximation Error}
We first look at the quantity $\abs{P^\star - D_\theta^\star}$. This can be further split as follows
\begin{equation}
    \begin{aligned}
        \abs{P^\star - D_\theta^\star} &= \abs{(P^\star - D^\star) + (D^\star - D_\theta^\star)}\\
        &\leq \abs{P^\star - D^\star} + \abs{D^\star - D_\theta^\star}
    \end{aligned}
\end{equation}
where $D^\star$ is the solution to the dual problem associated with \eqref{eq:functional_primal}. This is defined with the Lagrangian 
\begin{equation}
\label{eq:functional_lagrangian}
\mathcal{L}(\phi, \lambda ) =  \mathbb{E}_{s \sim \mathcal{S}_\mathcal{D}}[\ell_0(s, \phi)] + \lambda^\top \mathbb{E}_{s \sim \mathcal{S}_\mathcal{D}}[\boldsymbol{g}(x, \phi) ]
\end{equation}
and the dual problem
\begin{equation}
\label{eq:functional_dual}
D^\star = \max_{\lambda \geq 0} \min_{\phi \in \Phi} \mathcal{L}(\phi, \lambda).
\end{equation}

We note that the quantity $\abs{P^\star - D^\star}$ is actually 0 due to a result from \cite[Theorem 1]{ribeiro2012optimal}. The theorem is reproduced here using the notation of this paper.
\begin{theorem}[\cite{ribeiro2012optimal}, Theorem 1]
\label{thm:ribeiro_zero_duality}
    There is zero duality between the primal problem \eqref{eq:functional_primal} and dual problem \eqref{eq:functional_dual}, if
    \begin{enumerate}
        \item There exists a strictly feasible solution ($\phi$, $\lambda$) to \eqref{eq:functional_dual}
        \item The distribution $S$ is nonatomic.
    \end{enumerate}
\end{theorem}
While the problem defined in \cite{ribeiro2012optimal} is different from the sufficiently accurate problem defined in \ref{eq:functional_primal}, there is an equivalent problem formulation (see Appendix \ref{appendix:equivalent_formulation}). Since Theorem \ref{thm:empirical_duality_gap} fulfills the assumptions of Theorem \ref{thm:ribeiro_zero_duality}, we get $\abs{P^\star - D^\star} = 0$.

For the second half of this approximation error, $\abs{D^\star - D_\theta^\star}$, has also been previously studied in \cite[Theorem 1]{eisen2018learning} in the context of a slightly different problem formulation. The following theorem adapts \cite[Theorem 1]{eisen2018learning} to the \textit{Sufficiently Accurate} problem formulation \eqref{eq:functional_primal}.
\begin{theorem}
\label{thm:eisen_parameterization}
Given the primal problem \eqref{eq:functional_primal} and the dual problem \eqref{eq:parameterized_dual}, along with the following assumptions
\begin{enumerate}
    \item $\Phi_\theta$ is an $\epsilon$-universal function approximator for $\Phi$, and there exists a strictly feasible solution $\phi_\theta$ for \eqref{eq:parameterized_primal}.
    \item The loss and constraint functions are expectation-wise Lipschitz-continuous with constant $L$.
    \item All assumptions of Theorem, \ref{thm:ribeiro_zero_duality}
\end{enumerate}
The dual value, $D_\theta^\star$ is bounded by
\begin{equation}
    D^\star \leq D_\theta^\star \leq D^\star + (\norm{\lambda_{L\epsilon}^\star}_1 + 1) L \epsilon,
\end{equation}
where $\lambda_{L\epsilon}^\star$ is the dual variable that achieves the optimal solution to \eqref{eq:perturbed_dual}.
\end{theorem}
\begin{proof}
See Appendix \ref{appendix:eisen_proof}
\end{proof}

Again, the assumptions of Theorem \ref{thm:empirical_duality_gap} fulfill the assumptions for Theorem \ref{thm:eisen_parameterization}.
Due to notational differences, as well as a different way of framing the optimization problem, the proof has been adapted from \cite{eisen2018learning} and is given in Appendix \ref{appendix:eisen_proof}.
With Theorem \ref{thm:ribeiro_zero_duality} and \ref{thm:eisen_parameterization}, the following can be stated
\begin{equation}
\label{eq:parameterized_bound}
    \abs{P^\star - D_\theta^\star} \leq (\norm{\lambda_{L\epsilon}^\star}_1 + 1) L \epsilon
\end{equation}

\subsection{Empirical Error}
We now look at the empirical error, $\abs{D_\theta^\star - D_N^\star}$. We first observe the following Lemma.
\begin{lemma}
\label{lemma:empirical_error}
Let $\Delta \mathcal{L} (\theta,\lambda) = |\mathcal{L}_\theta(\theta, \lambda) - \mathcal{L}_N(\theta, \lambda) |$. Then under the assumption of Theorem \ref{thm:empirical_duality_gap} it follows that 
\begin{equation}
\label{eq:empirical_bound}
|D_\theta^\star - D_N^{\star} | \leq \sup_{\theta,\lambda}\Delta \mathcal{L}(\theta, \lambda).
\end{equation} 
\end{lemma}
\begin{proof}
See Appendix \ref{appendix:lemma_proof}
\end{proof}

\subsection{Probabilistic Bound}
Substituting the parameterized bound \eqref{eq:parameterized_bound} and the empirical bound \eqref{eq:empirical_bound} in \eqref{eq:full_bound} yields the following implication
\begin{equation}\label{eq:implication}
     \sup_{\theta,\lambda}\Delta \mathcal{L}(\theta, \lambda) \leq \delta -\epsilon L(\norm{\lambda_{L\epsilon}^\star}_1 + 1) \Rightarrow \norm{P^\star-D_N^\star} \leq \delta.
\end{equation}
Let $\mathbb{P}\left(|P^\star - D_N^\star| \leq {\delta}\right)$ be a probability over samples $\{s_1, s_2, \hdots, s_N\}$ that are drawn to estimate the expectation in the primal problem  \eqref{eq:empirical_primal}. Using the implication \eqref{eq:implication} it follows that 

\begin{equation}
\begin{aligned}
    &\mathbb{P}\left(|P^\star - D_N^\star| \leq {\delta}\right) \\
    &\geq \mathbb{P}\left(\sup_{\theta,\lambda}\Delta \mathcal{L}(\theta, \lambda)  \leq {\delta} - \epsilon L (\norm{\lambda_{L\epsilon}^\star}_1 + 1)\right) \\
    &= 1 - \mathbb{P}\left( \sup_{\theta,\lambda}\Delta \mathcal{L}(\theta, \lambda)  > {\delta} - \epsilon L (\norm{\lambda_{L\epsilon}^\star}_1 + 1)\right),
\end{aligned}    
\end{equation}

where the equality follows directly from the fact that for any event $A$, $\mathbb{P}(A) = 1 - \mathbb{P}(A^c)$. The assumptions of Theorem \ref{thm:empirical_duality_gap} allows us to use the following result from Statistical Learning Theory \cite[(Section II.B)]{vapnik1999overview}, 
\begin{equation}
    \lim_{N\rightarrow \infty} \mathbb{P}\left(\sup_{\theta,\lambda}\Delta \mathcal{L}(\theta, \lambda)  >{\delta} - \epsilon L (\norm{\lambda_{L\epsilon}^\star}_1 + 1) \right) = 0.
\end{equation}
Note that this theorem requires bounded loss functions. The assumptions for a bounded dual variable, and compact state-action space in Theorem \ref{thm:empirical_duality_gap} satisfies this constraint.

Thus, this establishes that for any $\delta>0$, we have 
 $   \lim_{N\rightarrow \infty} \mathbb{P}\left(|P^\star - D_N^\star| \leq {\delta}\right) = 1.$ 
This concludes the proof of the theorem. 

\section{Constrained Solution Via Primal-Dual Method} \label{sec:primal_dual}

\begin{figure}[h]
	\centering
	\includegraphics[width=0.4\textwidth]{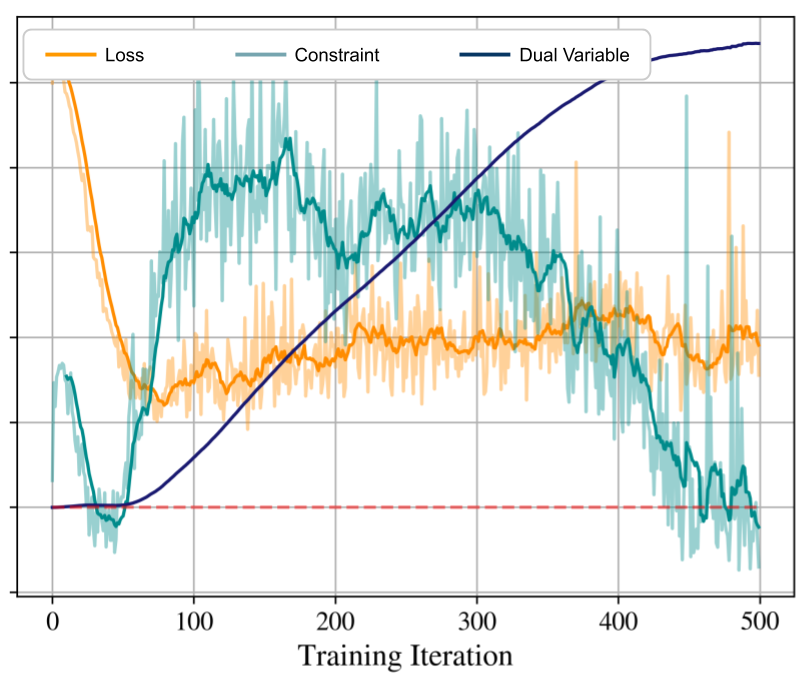}
	\caption{\textbf{Primal Dual Training Curve: } Example of the evolution of the constraint, loss, and dual variables during training. The red dotted line shows 0. The curves have been scaled so that they fit on the same y-axis and they are of different magnitudes. The constraint function is unfeasible, which leads to a growing dual variable. This can cause the loss function to increase until the constraint is feasible again.
	}
    \label{fig:training_curve}
\end{figure}

\begin{algorithm2e}
	\SetAlgoLined
	\KwIn{Data Samples, $S = \{ s_1, s_2, \hdots, s_N \}$\\
		\quad Neural Network with parameter space $\Theta$\\
		\quad Batch Size, $M$\\
		\quad Primal Learning rate, $\alpha_\theta$\\
		\quad Dual Learning rate, $\alpha_\lambda$ }
	Initialize $\theta \in \Theta$\;
	Initialize $\lambda = \boldsymbol{0}$\;
	\While{Not Converged}{
		Sample batch of data $\hat{S} = \{ s_{i_1}, \hdots, s_{i_M} \}$ from $S$\;
		Use $\hat{S}$ to compute estimates of $\nabla_\theta \mathcal{L}_N(\theta, \lambda)$ and $\nabla_\lambda \mathcal{L}_N(\theta, \lambda)$ (See \eqref{eq:primal_gradient} and \eqref{eq:dual_gradient})\;
		$\theta \leftarrow \theta - \alpha_\theta \nabla_\theta \mathcal{L}_N(\theta, \lambda)$ \;
		$\lambda \leftarrow \lambda + \alpha_\lambda \nabla_\lambda \mathcal{L}_N(\theta, \lambda)$ \;
		$\lambda \leftarrow \lbrack \lambda \rbrack_+$
	}
	\caption{Primal-Dual}
	\label{algo:primal_dual}
\end{algorithm2e}

Section \ref{sec:duality_gap} has shown that problem \eqref{eq:empirical_dual} can approximate \eqref{eq:functional_primal} given a large enough neural network and enough samples. This section will discuss how to compute a solution \eqref{eq:empirical_dual}. There are many primal-dual methods \cite{goldstein2014fast, gay1998primal, gill2012primal} in the literature to solve this exact problem, and Algorithm \ref{algo:primal_dual} is an example of a simple primal-dual algorithm. One way to approach this problem is to consider the optimal dual variable, $\lambda_N^\star$. Given knowledge of $\lambda_N^\star$, the problem reduces to the following \textit{unconstrained} minimization
\begin{equation}
    D_N^* = \min_{\theta \in \Theta} \mathcal{L}_N(\phi, \lambda_N^\star)
\end{equation}
A possible solution method is to start with an estimate of $\lambda_N^\star$, and solve the minimization problem. Then holding the primal variables fixed, update the dual variables by solving the outer maximization. This method can be seen as solving a sequence of unconstrained minimization problems. This method can be further approximated; instead of fully minimizing with respect to the primal variables, a gradient descent step can be taken. And instead of fully maximizing with respect to the dual variables, a gradient ascent step can be taken. This leads to Algorithm \ref{algo:primal_dual} where we iterate between the inner minimization step and the outer maximization step. At each iteration, dual variables are projected onto the positive orthant of $\mathbb{R}^K$, denoted by the projection operator, $\lbrack \lambda \rbrack_+$. This is to ensure non-negativity of the dual variables.

\begin{figure*}[ht]
    \centering
    \includegraphics[width=0.9\textwidth]{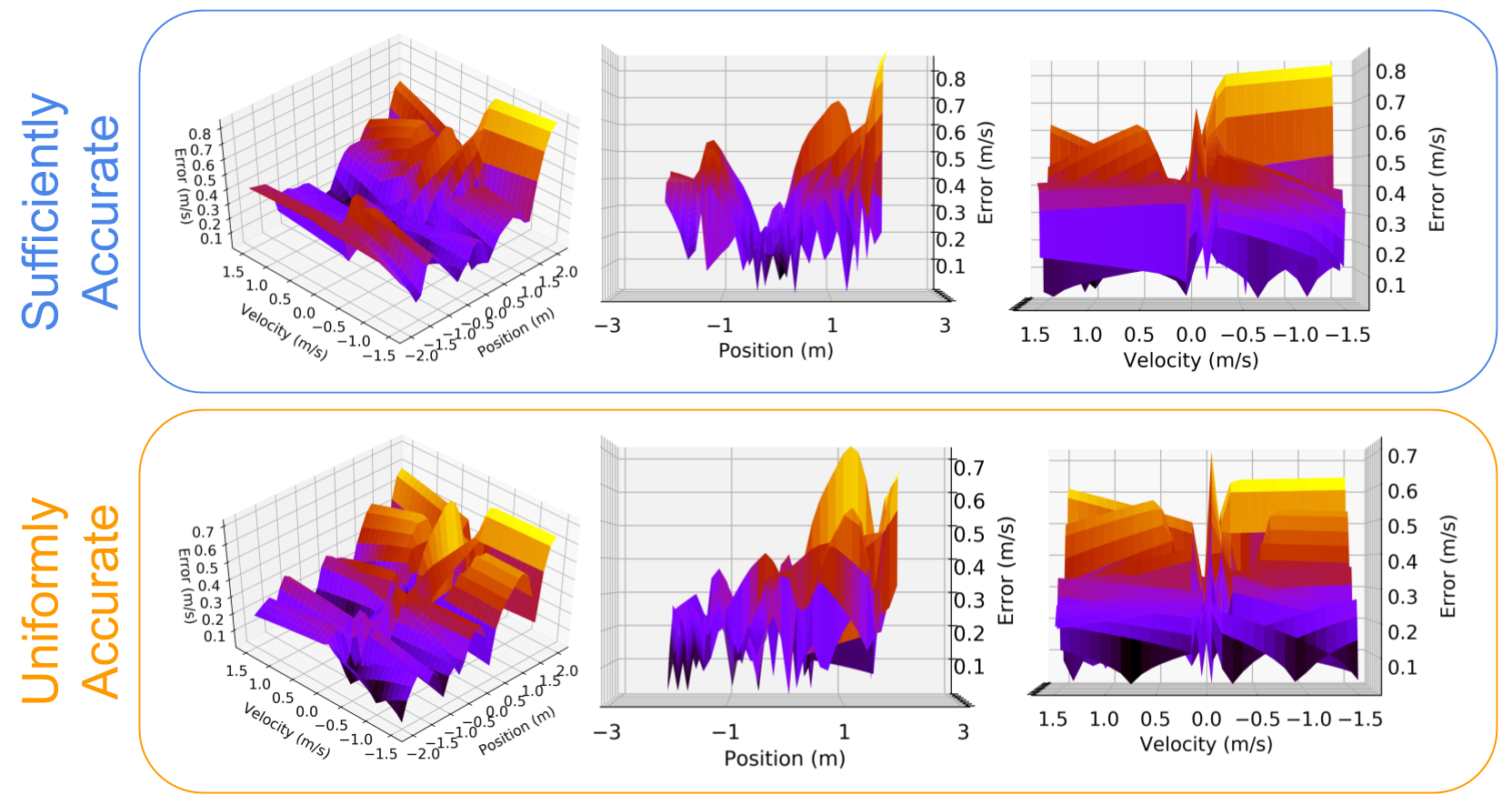}
    \caption{\textbf{Double Integrator model errors: } The error in predicted velocity of the sufficiently accurate and the uniformly accurate models. There is a constant control input of $u=1$ used to generate these plots. The top row contains three different views of the error for the sufficiently accurate model, while the bottom row contains the same three views for the uniformly accurate model.
    The z-axis on each plot is the error in the velocity for the difference $|f(x_t, u_t) - (\hat{f}(x_t, u_t) + \phi_\theta(x_t, u_t))|$. Best viewed in color.}
    \label{fig:double_integrator_model_errors}
\end{figure*}

In many cases, the full gradient of $\nabla_\theta \mathcal{L}_N(\theta, \lambda)$ and $\nabla_\lambda \mathcal{L}_N(\theta, \lambda)$ can be too expensive to compute. This is due to the possibly large number of samples $N$. An alternative is to take stochastic gradient descent and ascent steps. The gradients can be approximated by taking $M$ random samples of the whole dataset $S = {s_1, \hdots, s_N}$. The samples will be denoted as $\hat{S} = {s_{i_1}, \hdots, s_{i_M}}$ where $i_1$ is an integer index into whole dataset $S$. Using $\hat{S}$, we obtain
\begin{equation}
\label{eq:primal_gradient}
\begin{aligned}
    \nabla_\theta \mathcal{L}_N(\theta, \lambda) &= \nabla_\theta \lbrack \frac{1}{N} \sum_{i=1}^N \ell_0(s_i, \phi_\theta) + \lambda^\top \frac{1}{N}\sum_{i=1}^N \boldsymbol{g}(s_i, \phi_\theta) \rbrack \\
    &\approx \nabla_\theta \lbrack \frac{1}{M} \sum_{j=1}^M \ell_0(s_{i_j}, \phi_\theta) + \lambda^\top \frac{1}{M}\sum_{j=1}^M \boldsymbol{g}(s_{i_j}, \phi_\theta) \rbrack \\
    &= \frac{1}{M} \sum_{j=1}^M \nabla_\theta \ell_0(s_{i_j}, \phi_\theta) + \lambda^\top \frac{1}{M}\sum_{j=1}^M \nabla_\theta \boldsymbol{g}(s_{i_j}, \phi_\theta).
\end{aligned}
\end{equation}
The gradients $\nabla_\theta \ell_0(s_{i_j}, \phi_\theta)$ and $\nabla_\theta \boldsymbol{g}(s_{i_j}, \phi_\theta)$ can be computed easily using backpropogation. Similarly, for $\nabla_\lambda \mathcal{L}_N(\theta, \lambda)$, 
\begin{equation}
\label{eq:dual_gradient}
\begin{aligned}
    \nabla_\lambda \mathcal{L}_N(\theta, \lambda) 
    &\approx \nabla_\lambda \lbrack \frac{1}{M} \sum_{j=1}^M \ell_0(s_{i_j}, \phi_\theta) + \lambda^\top \frac{1}{M}\sum_{j=1}^M \boldsymbol{g}(s_{i_j}, \phi_\theta) \rbrack\\
    &= \frac{1}{M}\sum_{j=1}^M \boldsymbol{g}(s_{i_j}, \phi_\theta).
\end{aligned}
\end{equation}
The dual gradient can be estimated as simply the average of the constraint functions over the sampled dataset.

In the simplest form of the primal-dual algorithm, the variables are updated with simple gradient ascent/descent steps. These updates can be replaced with more complicated update schemes, such as using momentum \cite{sutskever2013importance} or adaptive learning rates \cite{kingma2014adam}. Higher order optimization methods such as Newton's method can be used to replace the gradient ascent and descent steps. For large neural networks, this can be unfeasible as it requires the computation of Hessians with respect to neural network weights. The memory complexity for the Hessian is quadratic with the number of neural network weights.

The primal-dual algorithm presented here is not guaranteed to converge to the global optimum. With proper choice of learning rate, it can converge to a local optimum or saddle point. This issue is present in unconstrained formulations like \eqref{eq:unconstrained_problem} as well. An example of the evolution of the loss and constraint functions is shown in Figure \ref{fig:training_curve}.

\section{Experiments}\label{sec:experiments}
This section shows examples of the \textit{Sufficiently Accurate} model learning problem. First, experiments are performed using a simple double integrator experiencing unknown dynamic friction. The simplicity of this problem along with the small state space allows us to explore and visualize some of the properties of the approximated solution. Next, two more interesting examples are shown. One example learns how a ball bounces on a paddle with unknown paddle orientation and coefficient of restitution. The other example mitigates ground effects which can disturb the landing sequence of a quadrotor. The experiments will compare the sufficiently accurate problem \eqref{eq:functional_primal} with the unconstrained problem \eqref{eq:unconstrained_problem} which will be denoted as the \textit{Uniformly Accurate} problem. Each experimental subsection will be broken down into three parts, 1) System and Task introduction, 2) Experimental details, and 3) Results.

\subsection{Double Integrator with Friction}

\begin{figure}[h]
	\centering
	\includegraphics[width=0.4\textwidth]{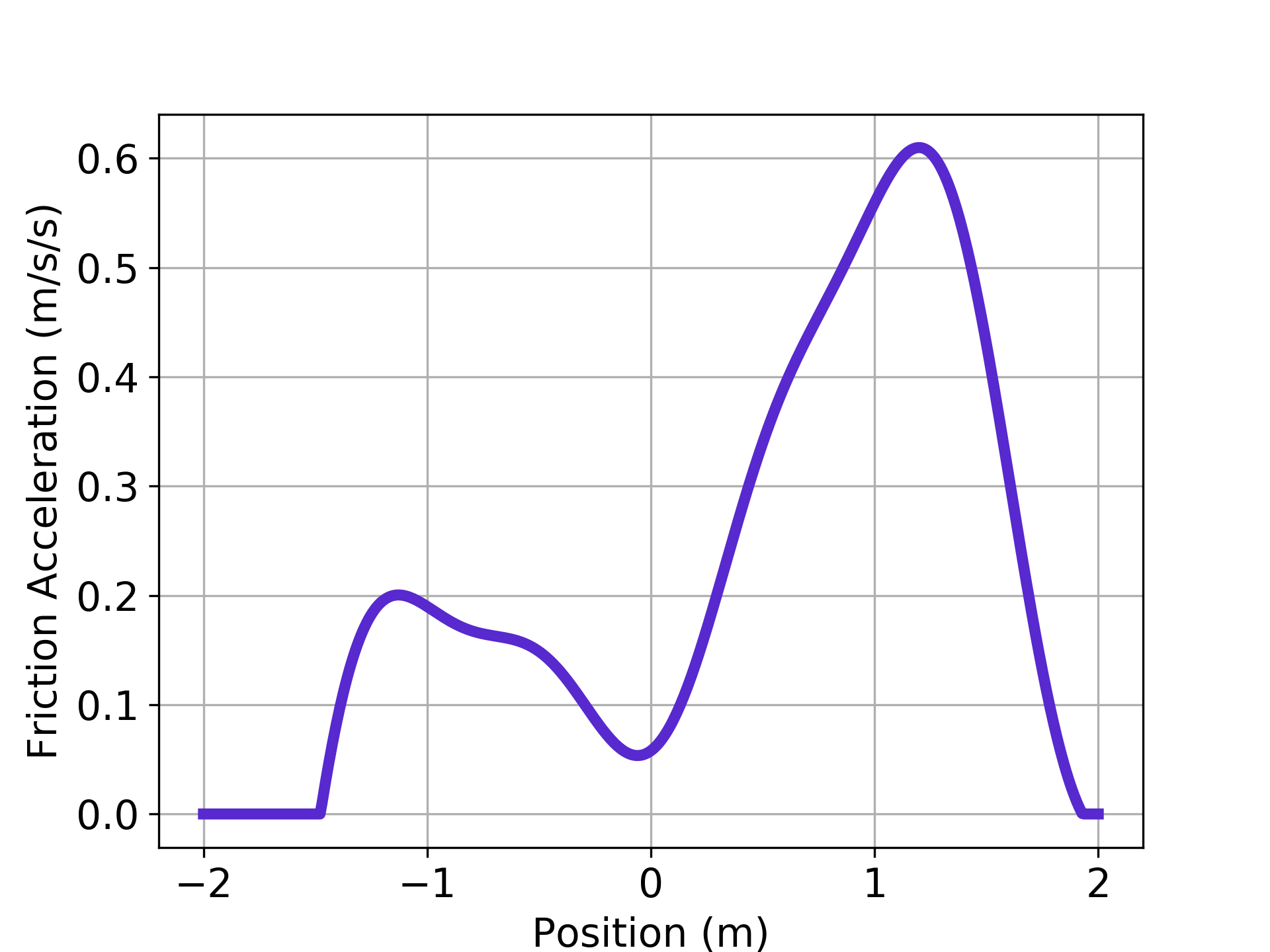}
	\caption{\textbf{Double integrator friction force: } The magnitude of the acceleration due to the position varying kinetic friction force.
	}
    \label{fig:friction}
\end{figure}

\begin{figure}[h]
	\centering
	\includegraphics[width=0.5\textwidth]{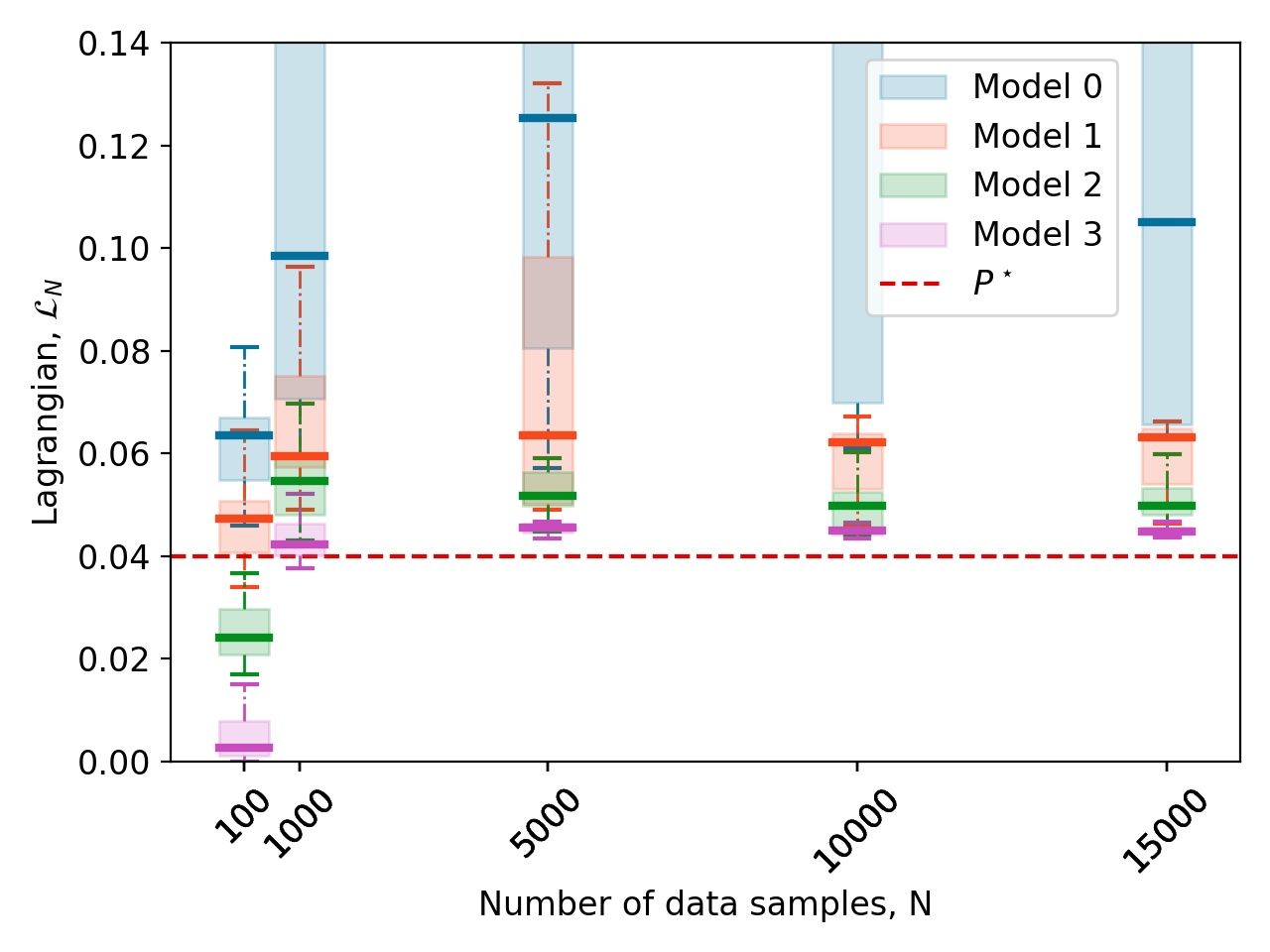}
	\caption{\textbf{Experimental Surrogate Duality Gaps:} Duality gaps of learning the double integrator problem with different neural networks and sample sizes (best viewed in color). The y-axis shows the Lagrangian value at the end of training which approximates $D_N^\star$. The models numbers indicate how large the neural network is with Model 0 being the smallest network. For each $N$, 15 tests with each model were run, and a box plot is shown that indicates the median as a solid bolded line. The ends of each box are the 1st and 3rd quartile, while the whiskers on the plot are the minimum and maximum values. The red line is the optimal value to the original problem \eqref{eq:functional_primal}. Note that for $N=10000$, the median is not shown for Model 0 as it is very large (0.21).
	}
    \label{fig:double_integrator_converging_means}
\end{figure}

\begin{figure}[h]
	\centering
	\includegraphics[width=0.4\textwidth]{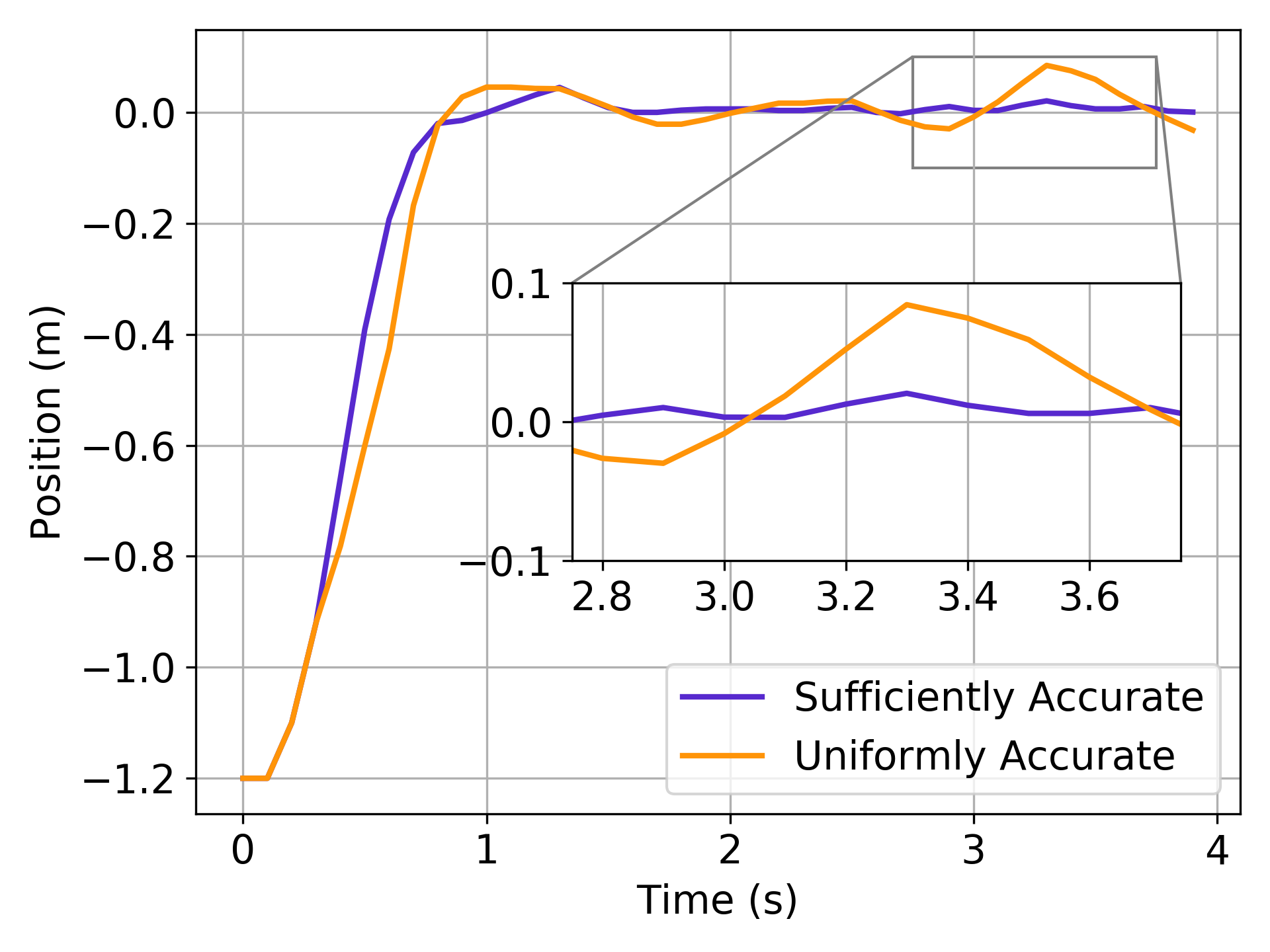}
	\caption{\textbf{Trajectory of the double integrator position:} This shows the evolution of the trajectory of the double integrator when controlled using MPC with both a sufficiently accurate and uniformly accurate model.
	}
	\label{fig:double_integrator_trajectory}
\end{figure}

\subsubsection{Introduction}
To analyze aspects of the \textit{Sufficiently Accurate} model learning formulation, simple experiments are performed on a simple double integrator system with dynamic friction. When trying to control the system, a simple base model to use is that of a double integrator without any friction
\begin{equation}
    \begin{bmatrix} p_{t+1} \\ v_{t+1} \end{bmatrix} = 
    \begin{bmatrix} 1 & \Delta t \\ 0 & 1  \end{bmatrix} \begin{bmatrix} p_{t} \\ v_{t} \end{bmatrix} + \begin{bmatrix} 0 \\  \Delta t \end{bmatrix} u_t
    \label{eq:double_integrator_base}
\end{equation}
where $p$ is the position of the system, $v$ is the velocity, $u$ is the control input, and $\Delta t$ is the sampling time. The state of the system is $x = \lbrack p, v \rbrack$.The true model of the system that is unknown to the controller is
\begin{align}
    \begin{bmatrix} p_{t+1} \\ v_{t+1} \end{bmatrix} =
    \begin{bmatrix} 1 & \Delta t \\ 0 & 1  \end{bmatrix} \begin{bmatrix} p_{t} \\ v_{t} \end{bmatrix} + \begin{bmatrix} 0 \\ (u_t \Delta t) - c(v_t, u_t, b(p_t)) \end{bmatrix}
    \label{eq:double_integrator_with_friction}
\end{align}
where $b(p_t)$ a position varying kinetic friction. $c$ is a function that ensures that the friction cannot reverse the the direction of the speed (it is an artifact of the discrete time formulation)
\begin{equation}
    c(v_t, u_t, b(p_t)) = 
    \begin{cases}
    v_t + u_t \Delta t, \text{ if }\\
        \quad sign(v_t + u_t \Delta t) \neq \\
        \quad sign(v_t + u_t \Delta t + b(p_t))\\
    b(p_t), \text{ otherwise}.
    \end{cases}
\end{equation}
If within a single time step, the friction force will change the sign of the velocity, $c$ will set $v_{t+1}$ to be 0. Otherwise, $c$ will not modify the friction force in any way.
The specific $b(p)$ used is shown in Figure \ref{fig:friction} and the sampling time is set to $\Delta t = 0.1$. The task is to drive the system to the origin $\begin{bmatrix} p & v \end{bmatrix} = \begin{bmatrix} 0 & 0 \end{bmatrix}$. 

\subsubsection{Experimental Details}\label{sec:double_integrator_experimental_details}
The goal of model learning in this experiment is to learn $\phi_\theta(x, u)$ such that $f(x, u) \approx \hat{f}(x, u) + \phi_\theta(x, u)$ where $f$ is \eqref{eq:double_integrator_with_friction} and $\hat{f}$ is \eqref{eq:double_integrator_base}. A \textit{Uniformly Accurate} model will be learned using \eqref{eq:unconstrained_problem} along with a \textit{Sufficiently Accurate} model using the problem defined in Example \ref{example:selective_accuracy}. 
In the scenario defined by \eqref{eq:selective_accuracy}, $\mathbb{I}(s)$ is active in the region $\{(p, v)\in\mathbb{R}^2 : \norm{ [p,v]^\top}_{\infty}\leq 0.5 \}$ and $\epsilon_c = 0.035$. The constraint, therefore, enforces a high accuracy in the state space near the origin. 

The neural network, $\phi_\theta$, used to approximate the residual dynamics has two hidden layers. The first hidden layer has four neurons, while the second has two. Both hidden layers use a parametric rectified linear (PReLU) activation  \cite{he2015delving}. The input into the network is a concatenated vector of $\lbrack p_t, v_t, u_t \rbrack$. The output layer's weights are initially set to zero so before learning the residual error, the network will output zero. The dataset used to train both the sufficiently and uniformly accurate models is generated by uniformly sampling $15,000$ positions from [-2, 2], velocities from [-2.5, 2.5], and control inputs from [-10, 10]. The real model \eqref{eq:double_integrator_with_friction} is then used to obtain the true next state. Instead of simple gradient descent/ascent, ADAM \cite{kingma2014adam} is used as an update rule with $\alpha_\theta = 1\times 10^{-3}$ and $\alpha_{\lambda} = 1\times 10^{-4}$. Both models were trained in 200 epochs.

The models are then evaluated on how well it performs within a MPC controller defined in \eqref{eq:double_integrator_mpc}. This controller seeks to drive the system to the origin while obeying control constraints. The controller is solved using a Sequential Quadratic Programming solver \cite{Kraft1988ASP} with a time horizon of $T=10$. The models are evaluated in 200 different simulations where $x_{start}$ is drawn uniformly from $\lbrack -2, 2 \rbrack$.

\begin{equation}
\begin{aligned}
    \min_{\{x_t, u_t \}_{t=1}^\top} & \sum_{t=1}^\top |x_t| \\
    \text{s.t. } & |u_t| \leq 10,  t = 1, \hdots, T \\
    & x_{t+1} = \hat{f}(x_t, u_t) + \phi_\theta(x_t, u_t), t = 1, \hdots, T-1 \\
    & x_1 = x_{start} \\
    & \dot{x}_1 = 0
\end{aligned}
\label{eq:double_integrator_mpc}
\end{equation}

\begin{figure*}[h]
    \centering
    \includegraphics[width=0.9\textwidth]{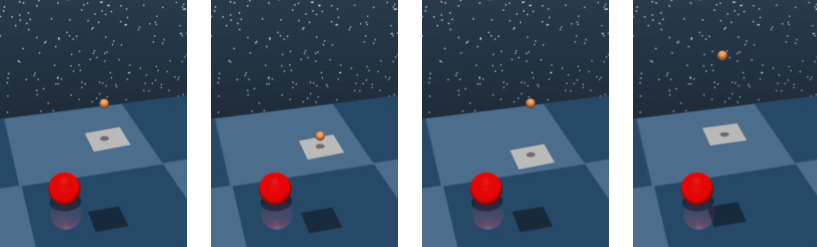}
    \caption{\textbf{Ball Bouncing Simulation: } The paddle seeks to bounce the orange ball above a target position on the xy plane, represented by the red ball. This figure shows a time sequence of a bounce from left to right.}
    \label{fig:ball_sequence}
\end{figure*}

\subsubsection{Results}
The sufficiently accurate formulation utilizes the prior knowledge that the model should be more accurate near the goal in order to stop efficiently. While the system is far from the origin, the control is simple, regardless of the friction; the controller only needs to know what direction to push in. A plot of the accuracy of both models is shown in Figure \ref{fig:double_integrator_model_errors} and summarized in Table \ref{table:double_integrator_state_space_errors}. It is noticeable that the \textit{Sufficiently Accurate} model has low average error near the origin, but suffers from higher average error outside of the region defined by $\mathbb{I}(s)$. This is the expected behavior. 

The performance of the controllers are summarized in Table \ref{table:double_integrator_controller_obj}.
Even though \textit{Sufficiently Accurate} has higher error outside of the constraint region and lower error within, it leads to lower costs when controlling the double integrator. The reason is shown in Figure \ref{fig:double_integrator_trajectory}, where the sufficiently accurate model may get to \textit{steady state} a bit slower but is able to control the overshoot better and not have oscillations near the origin. This is because the model is purposefully more accurate near the origin as it is more important for this task.

\begin{table}
  \centering
  \begin{tabular}{ | c | c | c |}
    \hline
    & $\sum_{t=1}^\top |x_t|$ & $\sum_{t=1}^\top |x_t| / |x_{start}|$ \\ \hline
    \cellcolor{black!20} Uniformly Accurate & \cellcolor{black!20} $5.51 \pm 3.46$ & \cellcolor{black!20} $5.66 \pm 0.89$ \\ \hline
    Sufficiently Accurate & $\bm{4.73 \pm 3.38}$ & $\bm{4.57 \pm 0.73}$ \\ \hline
  \end{tabular}
  \caption{\textbf{Controller performance: } This table shows the results of 200 trials of simulating the double integrator starting from different positions. Each entry shows the mean and one standard deviation. The first column shows the raw cost function of the MPC problem averaged over all trials. The second column shows an average normalized cost where each cost is normalized by the absolute value of the starting position. This is due to the fact that larger magnitude starting locations will have higher costs. }
  \label{table:double_integrator_controller_obj}
\end{table}

\begin{table}
  \centering
  \begin{tabular}{ | c | c | c | c| }
    \hline
    & All state space & $\mathcal{I}_K$ & $\mathcal{I}_K^C$ \\ \hline
    \cellcolor{black!20} Uniform & \cellcolor{black!20} $\bm{0.110 \pm 0.098}$ & $\cellcolor{black!20} 0.083 \pm 0.054$ & \cellcolor{black!20} $\bm{0.113 \pm 0.10}$ \\ \hline
    Sufficient & $0.136 \pm 0.126$ & $\bm{0.048 \pm 0.040}$ & $0.146 \pm 0.13$ \\ \hline
  \end{tabular}
  \caption{\textbf{Mean errors between the true model and the learned models in various state space subsets: } The error is the absolute difference between the true model and the learned models. Each entry shows the mean and one standard deviation. The first column shows the average error for the whole state space. The second column shows the average error for the state space near the origin, while the third column shows the average error for the complement of that set (states far from the origin). The errors are evaluated on a test set not seen during training.}
  \label{table:double_integrator_state_space_errors} 
\end{table}

\subsubsection{Convergence Experiments}
The double integrator is a simple system. This enables running more comprehensive tests to experimentally show some aspects of Theorem \ref{thm:empirical_duality_gap}. For this particular system, we will run one more experiment where 4 different neural network architectures were used. Each network has two hidden layers with PReLU activation, where the only difference is in the number of neurons in each layer. Denoting a network as (number of neurons in first layer, number of neurons in second layer), the network sizes used are: (2, 1), (4, 2), (8, 4), (16, 8). A set of values of the number of samples, $N$, are also chosen: $\{100, 1000, 5000, 10000, 15000 \}$. For each $N$, 15 random datasets are sampled, and each neural network is trained with each dataset using the Sufficiently Accurate objective described in Section \ref{sec:double_integrator_experimental_details}. There is one minor difference in how the data is collected; a zero mean Gaussian noise with $\sigma=0.2$ is added to $v_{t+1}$. With noisy observations of velocity, the optimal model that can be learned for \eqref{eq:functional_primal} will have an objective value of $P^\star = 0.04$. The results of training each neural network model with each random dataset is shown in Figure \ref{fig:double_integrator_converging_means}. Each boxplot in the figure shows the distribution of the final value of the Lagrangian, $\mathcal{L}_N$, at the ending of training. This is an approximation of $D_N^\star$. The primal-dual algorithm may not be able to solve for the optimal $D_N^\star$, but the expectation is that for a simple problem like double integrator, the solution is somewhat close. In fact, Figure \ref{fig:double_integrator_converging_means} shows that with increasing model sizes and larger $N$, the distribution of the solutions appear to be converging to $P^{\star}$. Note that the figure shows the value of the Lagrangian with training data. Thus for small $N$, networks can overfit and have a near zero Lagrangian value. When increasing $N$, the networks have less of a chance to overfit to the training data.

\subsection{Ball Bouncing} \label{sec:ball_bouncing_experiment}

\begin{figure}[h]
	\centering
	\includegraphics[width=0.475\textwidth]{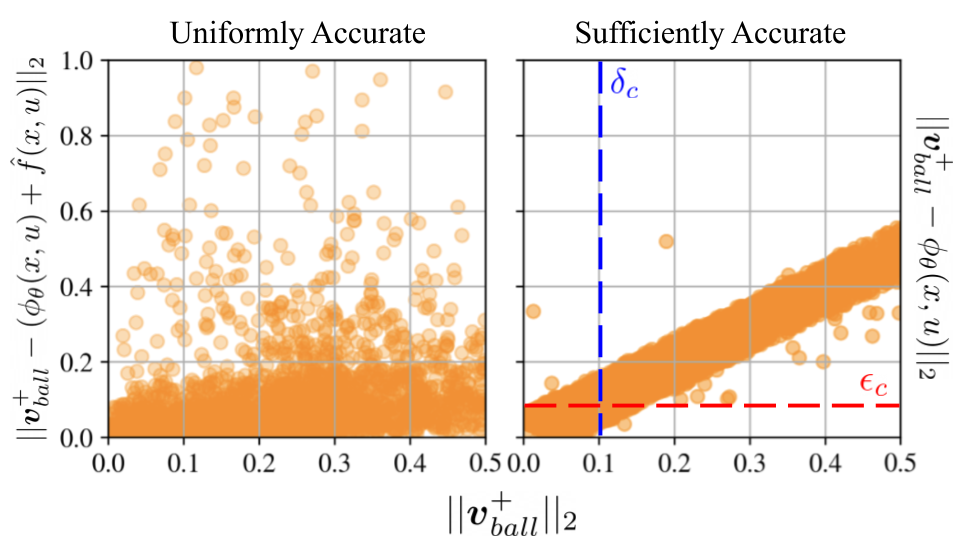}
	\caption{\textbf{Model errors vs. velocity magnitude: } The scatter plot shows the distribution of model errors versus the magnitude of the velocity of the true result. Both the sufficiently and uniformly accurate models are evaluated using a validation set that is not used during training. The blue dotted line represents the boundary of where the constraint set is and the red dotted line represents the boundary of the constraint function.}
	\label{fig:ball_model_errs}
\end{figure}

\begin{figure}[h]
	\centering
	\includegraphics[width=0.475\textwidth]{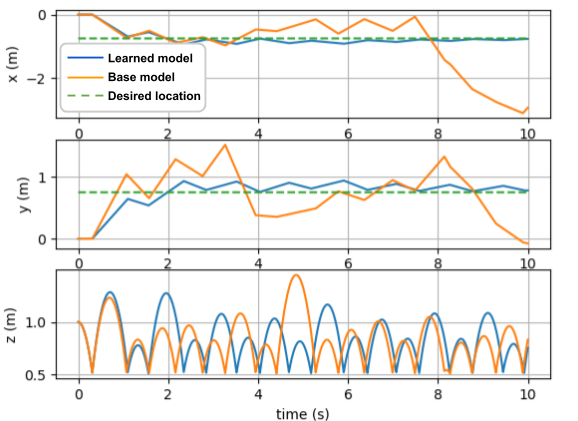}
	\caption{\textbf{Ball bouncing trajectory: } The (x, y z) trajectory of the ball is plotted for both the base model (with wrong parameters) and a learned model using the sufficiently accurate objective. This plot shows that the base model is not sufficient by itself to bounce the ball at a desired location.}
	\label{fig:ball_trajectory}
\end{figure}

\subsubsection{Introduction}
This experiment involves bouncing a ball on a paddle as in Figure \ref{fig:ball_sequence}. 
The ball has the state space $x = \lbrack \boldsymbol{p}_{ball}, \boldsymbol{v}_{ball} \rbrack$,
where $\boldsymbol{p}_{ball}$ is the three-dimensional position of the ball, $\boldsymbol{v}_{ball}$ is the three-dimensional velocity of the ball. The control input is $u = \lbrack \boldsymbol{v}_{paddle}, \boldsymbol{n}\rbrack$ where $\boldsymbol{v}_{paddle}$ is the velocity of the paddle at the moment of impact with the ball and $\boldsymbol{n}$ is the normal vector of the paddle, representing its orientation. This control input is a high level action and is realized by lower level controllers that attempt to match the velocity and orientation desired for the paddle.
A basic model of how the velocity of the ball changes during collision is
\begin{equation}
\label{eq:ball_base_model}
\begin{aligned}
    \boldsymbol{v}_{ball}^+ &= \alpha_{r} (\boldsymbol{v}_{rel}^- - 2\boldsymbol{n}(\boldsymbol{n} \cdot \boldsymbol{v}_{rel}^-)) + \boldsymbol{v}_{paddle} \\
    \boldsymbol{v}_{rel}^- &= (\boldsymbol{v}_{ball}^- - \boldsymbol{v}_{paddle})
\end{aligned}
\end{equation}
where the superscript $-$ refers to quantities before the paddle-ball collision and the superscript $+$ refers to quantities after the paddle-ball collision (the paddle velocity and orientation are assumed to be unchanged during and directly after collision). $\alpha_r$ is the coefficient of restitution. In this experiment, a neural network is tasked to learn the model of how the ball velocity changes, i.e. \eqref{eq:ball_base_model}.

\subsubsection{Experimental Details}
First, a neural network is trained without knowledge of any base model of how the ball bounces. This will be denoted as learning a \textit{full} model as opposed to a \textit{residual} model. This network is trained two ways, with the \textit{Uniformly Accurate} problem \eqref{eq:unconstrained_problem} as well as the \textit{Sufficiently Accurate} problem realized in Example \ref{example:normalized_objective}. The constants used in Example \ref{example:normalized_objective} are defined here as $\epsilon_c = 0.1$ and $\delta_c = 0.1$. 

A second neural network is trained for both the uniformly and sufficiently accurate formulations that utilizes the base model, $\hat{f}(x, u)$ given in \eqref{eq:ball_base_model} to learn a residual error. In the base model, the coefficient of restitution, $\alpha_r$, is wrong and the control $\boldsymbol{n}$ has a constant bias where a rotation of 0.2 radians is applied to the \textit{y-axis}. This is to simulate a robot arm picking up the paddle and not observing the rotation from the hand to the paddle correctly. 

The neural network used for all models has 2 hidden layers with 128 neurons in each using the pReLU activation. The input into the network is the the state of the ball and the control input at time of collision, $\lbrack x^-, u \rbrack$, and it outputs the ball velocity after the collision, $\boldsymbol{v}_{ball}^+$.  The network was trained using the ADAM optimizer with an initial learning rate of $10^{-3}$ for both the primal and dual variables. The data used for all model training was gathered by simulating random ball bounces in MuJoCo for the equivalent of 42 minutes in real life. 

All learned models are then evaluated with how well a controller utilizes them. The controller will attempt to bounce the ball at a specific $xy$ location. This is represented through the following optimization problem that the controller solves
\begin{equation}
\label{eq:ball_paddle_opt}
\begin{aligned}
&\min_{ {u}} \abs{ loc(\phi_\theta( {x,u})) -  {loc}_{desired} } \\
& \text{s.t. }  {roll}_{min} \leq  roll \leq  {roll}_{max} \\
& \quad {pitch}_{min} \leq  pitch \leq  {pitch}_{max}  \\
& \quad \boldsymbol{v}_{min} \leq \abs{\boldsymbol{v}_{rel} } \leq \boldsymbol{v}_{max}
\end{aligned}
\end{equation}
where $loc(\cdot)$ is a function that maps the velocity of the ball to the $xy$ location it will be in when it falls back to its current height. $roll$ and $pitch$ are both derived from the paddle normal $\boldsymbol{n}$. $\lbrack {loc}_{desired}, roll_{min}, roll_{max}, pitch_{min}, pitch_{max}, v_{min}, v_{max} \rbrack$ are parameters of the controller that can be chosen. The system and controller is then simulated in MuJoCo \cite{todorov2012mujoco} using libraries from the DeepMind Control Suite \cite{tassa2018deepmind}.

Each model is evaluated 500 different times for varying controller parameters. $loc_{desired}$ is uniformly distributed  in the region $\{(x, y) | -1m \leq x \leq 1m, -1m \leq y \leq 1m\}$, $\boldsymbol{v}_{min}$ uniformly sampled from the interval $[3m/s, 4m/s)$, and $\boldsymbol{v}_{max}$ is selected to be above $\boldsymbol{v}_{min}$ by between $1m/s$ to $2m/s$. 

\subsubsection{Results}
A plot of the model errors are shown in Figure \ref{fig:ball_model_errs}. While the uniformly accurate model has errors that are distributed more or less uniformly across all magnitudes of ball velocity, the sufficiently accurate model has a clear linear relationship. This is expected from the normalized objective that is used which penalizes errors based on large the velocity of the ball is. Therefore, larger velocities can have larger errors with the same penalty as smaller velocities with small errors. 
\begin{figure*}[h]
    \centering
    \includegraphics[width=0.95\textwidth]{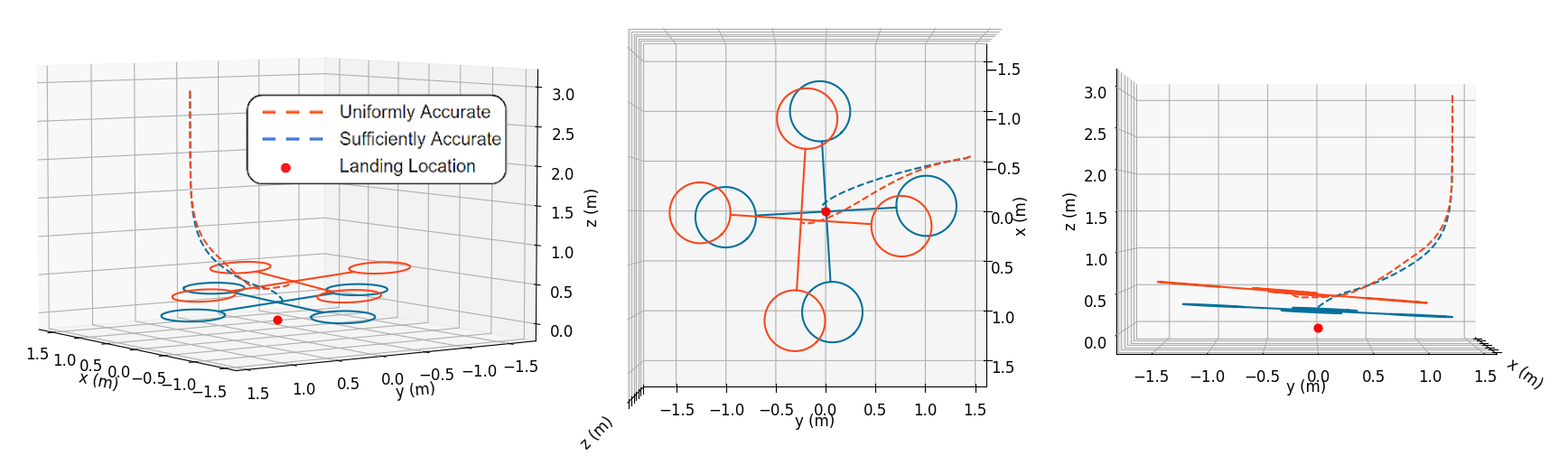}
    \caption{\textbf{3D Quadrotor Simulation: } Trajectory of the Sufficiently Accurate and Uniformly Accurate models used in a MPC controller (best viewed in color). The goal is to land at a precise point represented by the red dot. Each plot is a different view of the same data. The dotted lines represent the trajectory of the center of mass of the vehicle for 50 time steps. The state of the quadrotor at the last time step is drawn.}
    \label{fig:quad3d_traj}
\end{figure*}

The results of running each model with the controller 500 times is shown in Table \ref{table:ball_controller_errs}. The error characteristics of the \textit{Sufficiently Accurate} model (Figure \ref{fig:ball_model_errs}) allow it to out perform its \textit{Uniformly Accurate} counterpart with both a full model and a residual model. For the full model, the uniformly accurate problem yields a failure rate of over $20\%$ while the sufficiently accurate problem yields a failure rate of under $1\%$. Here, failure means the paddle fails to keep the ball bouncing. For the residual model, neither model failed because the base model provides a decent guess (though the base model by itself is not good enough for control, see Figure \ref{fig:ball_trajectory}). The sufficiently accurate model still provided better mean errors.

We hypothesize that the large errors spread randomly across the Uniform model leads to high variance estimates of the output given small changes in the input. For optimizers that use gradient information, this leads to a poor estimate of the gradient. For optimizers that are gradient free, this still causes problems due to the high variance of the values themselves. 

\begin{table}[h]
	\centering
	\begin{tabular}{c|c|c|c|}
		\cline{2-4}
		& & Uniform & Sufficient \\
		\cline{2-4}
		\multirow{2}{*}{Full Model}&Failure &$20.8 \%$ & $0.8 \%$\\ \cline{2-4}
								   &Mean error & $0.3136$ & \textbf{0.2124} \\
								   \cline{2-4}
								   \cline{2-4}
		\multirow{2}{*}{Residual Model}& \cellcolor{black!20} Failure & \cellcolor{black!20}$0 \%$ & \cellcolor{black!20} $0 \%$\\  \cline{2-4}
								   & \cellcolor{black!20} Mean error & \cellcolor{black!20} $0.164$ & \cellcolor{black!20} \textbf{0.156} \\
								   \cline{2-4}
	\end{tabular}
	\caption{\textbf{Ball bouncing controller performance: }Results of 500 trials of ball bouncing for each model. There is a full and residual model trained for both the uniformly accurate and sufficiently accurate model learning problems.}
	\label{table:ball_controller_errs}
\end{table}

\subsection{Quadrotor with ground effects} \label{sec:quadrotor_experiments}

\subsubsection{Introduction}
The last experiment deals landing a quadrotor while undergoing disturbances from ground effect. This disturbance occurs when a quadrotor operates near surfaces which can change the airflow \cite{sanchez2017characterization}. The state for the quadrotor model is a 12 degree of freedom model which consists of $x = \lbrack \boldsymbol{p}, \boldsymbol{v}, \boldsymbol{q}, \boldsymbol{\omega} \rbrack$ where $\boldsymbol{p} \in \mathbb{R}^3$ is position of the center of mass, $\boldsymbol{v}  \in \mathbb{R}^3$ is the center of mass velocity, $\boldsymbol{q} \in SO(3)$ is a quaternion that represents the orientation the quadrotor, and $\boldsymbol{w}  \in \mathbb{R}^3$ is the angular velocity expressed in body frame. The control input is $u = \lbrack u^{(1)}, u^{(2)}, u^{(3)}, u^{(4)} \rbrack$ where $u^{(i)}$ is the force from the $i^{\text{th}}$ motor.
The base model of the quadrotor, $\hat{f}(x, u)$ is as follows
\begin{equation}
\begin{aligned}
    \begin{bmatrix} \boldsymbol{p}_{t+1} \\ \boldsymbol{v}_{t+1} \\ \boldsymbol{q}_{t+1} \\ \boldsymbol{\omega}_{t+1} \end{bmatrix} &= 
    \begin{bmatrix} \boldsymbol{p}_{t} + \dot{\boldsymbol{p}_t} \Delta t  \\ \boldsymbol{v}_{t} + \dot{\boldsymbol{v}_t}\Delta_t \\ \frac{\boldsymbol{q}_{t} + \dot{\boldsymbol{q}_t}\Delta_t}{\norm{\boldsymbol{q}_{t} + \dot{\boldsymbol{q}_t}\Delta_t}_2} \\ \boldsymbol{\omega}_{t} + \dot{\boldsymbol{\omega}_t}\Delta_t \end{bmatrix}\\
    \begin{bmatrix} \dot{\boldsymbol{p}} \\ \dot{\boldsymbol{v}} \\ \dot{\boldsymbol{q}} \\ \dot{\boldsymbol{\omega}} \end{bmatrix} &= 
    \begin{bmatrix}
    \boldsymbol{v}\\
    \boldsymbol{q} \otimes \lbrack 0, 0, \sum_{i=1}^4 u^{(i)} / m \rbrack^\top \otimes \boldsymbol{q}^{-1} - \lbrack 0, 0, 9.81 \rbrack^\top\\
    \frac{1}{2} \boldsymbol{\omega} \otimes \boldsymbol{q}\\
    \mathcal{I}^{-1} (\boldsymbol{T} - \boldsymbol{\omega} \times (\mathcal{I} \boldsymbol{\omega}))
    \end{bmatrix} \\
    \boldsymbol{T} &= 
    \begin{bmatrix}
    u^{(4)} - u^{(2)}\\
    u^{(3)} - u^{(1)} \\
    (u^{(1)} + u^{(3)}) - (u^{(2)} + u^{(4)})
    \end{bmatrix}
\end{aligned}
\end{equation}
where $m$ is the total mass of the quadrotor (set to be 1kg for all experiments) and $\mathcal{I}$ is inertia matrix around the principle axis (set to be identity for all experiments). The $\times$ symbol represents cross product, and $\otimes$ represents quaternion multiplication. When using $\otimes$ between a vector and a quaternion, the vector components are treated as the imaginary components of a quaternion with a real component of 0. The discrete model normalizes the quaternion for each state update so that it remains a unit quaternion. The body frame of the quadrotor is such that the x axis aligns with one of the quadrotor arms, and the z axis points ``up."

The true model used in simulation adds disturbances to the force on each propeller, but is otherwise the same as the base model:
\begin{equation}
    f(x, u) = \hat{f}(x, h(x, u))
\end{equation}
where $h : \mathbb{R}^n \times \mathbb{R}^p \rightarrow \mathbb{R}^p$ is the ground effect model. In this experiment we provide a simplified model of ground effects where each motor has independence disturbances. The $i^{\text{th}}$ output of the ground effect model, $h_i(x, u)$ is
\begin{equation}
\begin{aligned}
    h_i(x, u) &= u^{(i)} (1 + K_{ground})\\
    K_{ground} &= (\lbrack 1 - \frac{h_{prop}}{h_{max}} \rbrack_+)  (\frac{4\lbrack\theta_{ground} - \frac{\pi}{2}\rbrack_+^2}{\pi^2}) \alpha 
\end{aligned}
\end{equation}
where $h_{prop}$ is height of the propeller above the ground (not the height of the center of mass), $h_{max}$ is a constant that determines the height at which the ground effect is no longer in effect. $\theta_{ground}$ is the angle between the unit vector aligned with the negative $z$ axis of the quadrotor and the unit vector $\lbrack0, 0, -1\rbrack$. $\alpha$ is a number in the set $\lbrack 0, 1 \rbrack$ that represents the maximum fraction of the propeller's generated force that can be added as a result of ground effect. As a reminder, the $\lbrack \cdot \rbrack_+$ operator projects its arguments onto the positive orthant. A visualization of $K_{ground}$ is shown in Figure \ref{fig:quad3d_ground_effect}. In the experiments, $h_{max}=1.5$, $\alpha = 0.5$.

\begin{figure}[h]
	\centering
	\includegraphics[width=0.475\textwidth]{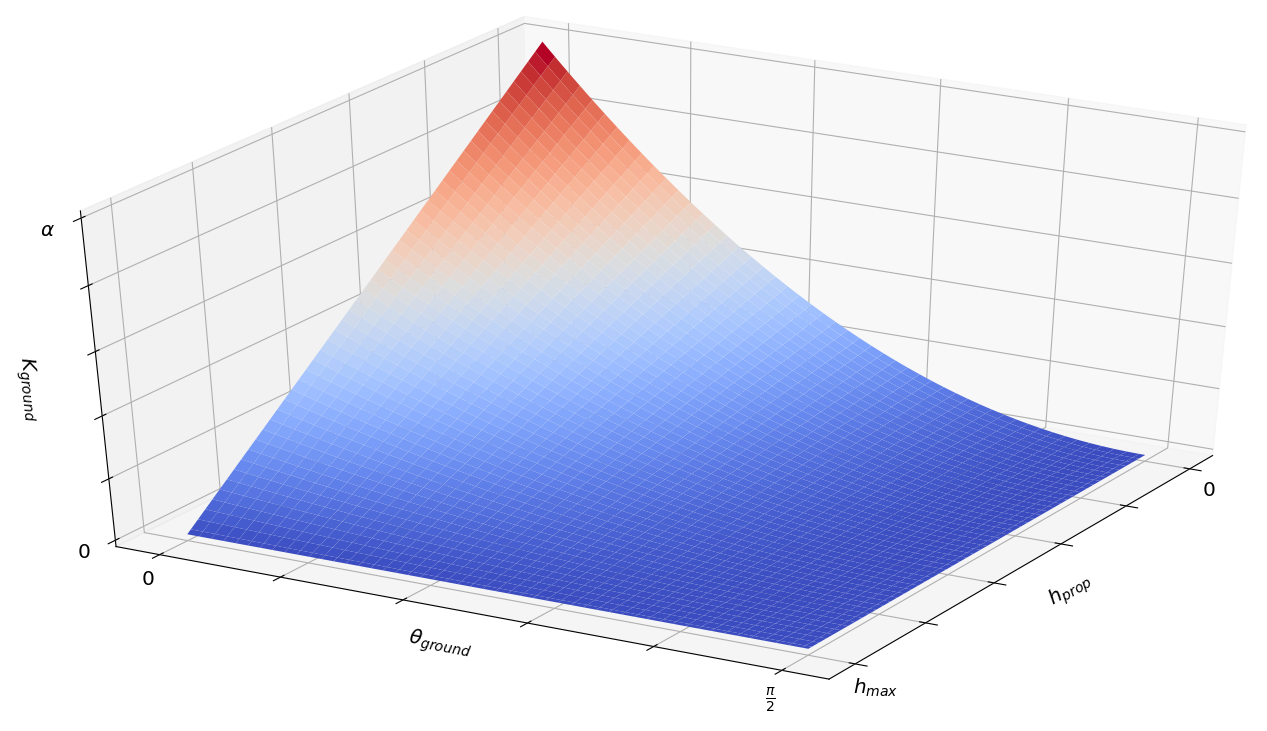}
	\caption{\textbf{Ground effect: } A visualization of $K_{ground}$ as a function of $\theta_{ground}$ and $h_{prop}$. }
	\label{fig:quad3d_ground_effect}
\end{figure}

\subsubsection{Experimental Details}
The \textit{Sufficiently Accurate} model trained using the problem presented in Example \ref{example:selective_accuracy}, where $\epsilon_c = 0.001$ and the indicator $\mathbb{I}(s)$ is active when the height of the quadrotor is less than $1.5$. 
A \textit{Uniformly Accurate} and \textit{Sufficiently Accurate} model are trained to learn the residual error between $f(x, u)$ and $\hat{f}(x, u)$. Both models use a neural network with 2 hidden layers of 16 and 8 neurons each with pReLU activation. The update for primal and dual variables used ADAM with $\alpha_\theta = 1\times 10^{-3}$ and $\alpha_{\lambda} = 1\times 10^{-4}$, and both models trained using $3,000$ epochs.
The training data consists of $10,000$ randomly sampled quadrotor states. The $(x,y)$ positions of the quadrotor were uniformly sampled from $\lbrack -10, 10 \rbrack$. The $z$ positions of the quadrotor were uniformly sampled from $\lbrack 0.1, 5 \rbrack$. Linear velocities components were uniformly sampled from $\lbrack -6, 6 \rbrack$. Angular velocities components were uniformly sampled from $\lbrack -2.5, 2.5 \rbrack$. Control inputs for each motor are sampled from $\lbrack 0, 10 \rbrack$. Quaternions are sampled by sampling random unit vectors along with a random angle in $\lbrack -0.4 rad, 0.4 rad\rbrack$. This angle-axis rotation is transformed into a quaternion.

Both models are tested by sampling a random starting location and asking the quadrotor to land at the origin. The controller used for landing is an MPC controller that repeatedly solves the following problem
\begin{equation}
\begin{aligned}
&\min_{\{u_t, x_t \}_{t=1}^T} \sum_{t=1}^T \norm{ \boldsymbol{p}_t - \boldsymbol{p}_{target}}_2 \\
&\text{s.t.  } 0 \leq u_t \leq 10, \forall t\\
&\quad x_{t+1} = \hat{f}(x_t, u_t) + \phi_\theta(x_t, u_t), t=1, \hdots, T-1\\
&\quad h_{t, com} \geq 0.2, \forall t\\
&\quad \abs{q_w - 1} \leq 0.05\\
\end{aligned}
\label{eq:quad3d_mpc}
\end{equation}
where $\boldsymbol{p}$ is the position of the quadrotor, $h_{t, com}$ is the height of the center of mass, and $q_w$ is the real component of the quaternion at the last time step. This problem encourages reaching a target, subject to control and dynamics constraint. It also has a constraint on the height of the quadrotor so it is always above a certain small altitude, and an orientation constraint on the last time step so it is mostly upright when it lands. 

\subsubsection{Results}
The results of running this controller over several different starting locations is shown in Table \ref{table:quad3d_results}. Similar to previous experiments, the \textit{Sufficiently Accurate} model has a higher loss overall, but better accuracy in the constrained area which is more important to the task. This allows the controller to utilize the higher accuracy to land the quadrotor precisely. An example of one of the landing trajectories is shown in Figure \ref{fig:quad3d_traj}. It can be seen that the Sufficiently Accurate model can more precisely land at the origin $(0, 0)$. It also is able to reach the ground faster, as it can more accurately compensate for the extra force caused by the ground surface. The ground effect can also disturb roll and pitch maneuvers which can offset the center of mass as well.

\begin{table}
  \centering
  \begin{tabular}{  | c | c | c | }
    \hline
    & Sufficiently Accurate & Uniformly Accurate \\
    \hline 
    \cellcolor{black!20} MPC cost & \cellcolor{black!20} $\bm{4.21 \pm 1.88}$ & \cellcolor{black!20} $4.55 \pm 1.90$ \\
    \hline
    $\mathbb{E}\lbrack l(s, \phi_\theta) \rbrack$ & $6.52 \times 10^{-4}$ & $\bm{6.16 \times 10^{-4}}$\\
    \hline
    \cellcolor{black!20} $\mathbb{E}\lbrack g(s, \phi_\theta) \mathbb{I}(s)\rbrack$ & \cellcolor{black!20} $\bm{1.00 \times 10^{-4}}$ & \cellcolor{black!20} $8.28 \times 10^{-4}$ \\
    \hline
  \end{tabular}
  \caption{\textbf{Quadrotor Results: } The first row shows the MPC cost average \eqref{eq:quad3d_mpc} and one standard deviation over 5 runs. The second row shows the training loss, while the third row shows the constraint function. The expected errors are evaluated on a test set not seen during training.}
  \label{table:quad3d_results}
\end{table}

\section{Conclusion}\label{sec:conclusion}
This paper presents \textit{Sufficiently Accurate} model learning as a way to incorporate prior information about a system in the form of constraints. In addition, it proves that this constrained learning formulation can have arbitrarily small duality gap. This means that existing methods like primal-dual algorithms can find decent solutions to the problem. 

With good constraints, the model and learning method can focus on important aspects of the problem to improve the performance of the system even if the overall accuracy of the model is lower. These constraints can come from robust control formulations or from knowledge of sensor noise. Some important questions to consider when using this method is how to choose good constraints. For some systems and tasks, it can be simple while for others, it can be quite difficult. This objective is not useful for all tasks and systems but rather for a subset of tasks and systems where prior knowledge is more easily expressed as a constraint.

\section*{Acknowledgements}
This material is based upon work supported by the National Science Foundation Graduate
Research Fellowship Program under Grant No. DGE-1321851. Any opinions,
findings, and conclusions or recommendations expressed in this material are those of the
authors and do not necessarily reflect the views of the National Science Foundation.

\FloatBarrier

\appendices
\section{Expectation-wise Lipschitz-Continuity of Loss functions}\label{appendix:continuity}

\subsection{$l(s, \phi) = \norm{\phi(x_t, u_t) - x_{t+1}}_2$}
The loss function $l(s, \phi) = \norm{\phi(x_t, u_t) - x_{t+1}}_2$ is expectation-wise Lipschitz-continuous in $\phi$.

Using the reverse triangle inequality, we get

\begin{equation}
\begin{aligned}
    &\norm{l(s, \phi_1) - l(s, \phi_2)}_\infty = \\
    &\quad \norm{\norm{\phi_1(x_t, u_t) - x_{t+1}}_2 - \norm{\phi_2(x_t, u_t) - x_{t+1}}_2}_\infty\\
    &\quad \leq \norm{ \norm{\phi_1(x_t, u_t) - \phi_2(x_t, u_t)}_2 }_\infty.
\end{aligned}
\end{equation}
By, the equivalence of norms, there exists $L$ such that $\norm{\phi_1(x, u) - \phi_2(x, u)}_2 \leq L \abs{\phi_1(x, u) - \phi_2(x, u)}$, for all $(x, u)$.
Thus,
\begin{equation}
    \norm{l(s, \phi_1) - l(s, \phi_2)}_\infty \leq L \norm{\phi_1 - \phi_2}_\infty
\end{equation}
Since this is true for any $s$, it is also true in expectation.
\begin{equation}
    \mathbb{E}_s \norm{l(s, \phi_1) - l(s, \phi_2)}_\infty \leq L \mathbb{E}_s \norm{\phi_1 - \phi_2}_\infty
\end{equation}

\subsection{$l(s, \phi) = \frac{\norm{ \phi(x_t, u_t) - x_{t+1} }_2}{ \norm{ x_{t+1} }_2} \mathbb{I}_{A}(s)$}
The loss function $l(s, \phi) = \frac{\norm{ \phi(x_t, u_t) - x_{t+1} }_2}{ \norm{ x_{t+1} }_2} \mathbb{I}_{A}(s)$ is also expectation-wise Lipschitz-continuous in $\phi$.
Following the same logic as for the euclidean norm, we get
\begin{equation}
\norm{l(s, \phi_1) - l(s, \phi_2)}_\infty \leq L \norm{ \frac{\mathbb{I}_A(s)}{\norm{x_{t+1}}_2} (\phi_1 - \phi_2)}_\infty 
\end{equation}
This reduces to
\begin{equation}
\norm{l(s, \phi_1) - l(s, \phi_2)}_\infty \leq L \norm{ \frac{1}{\min_{x_{t+1}} \norm{x_{t+1}}_2} (\phi_1 - \phi_2)}_\infty 
\end{equation}
when considering the case where $s \in A$. The largest that $\frac{\mathbb{I}_A(s)}{\norm{x_{t+1}}_2}$ can be is $1$ over the smallest value of $\norm{x_{t+1}}_2$. If $s \not \in A$, both sides reduce to 0, as the indicator variable is 0. This leads to
\begin{equation}
\begin{aligned}
\norm{l(s, \phi_1) - l(s, \phi_2)}_\infty &\leq L \norm{ \frac{1}{\epsilon} (\phi_1 - \phi_2)}_\infty \\
\mathbb{E} \norm{l(s, \phi_1) - l(s, \phi_2)}_\infty &\leq \frac{L}{\epsilon} \mathbb{E} \norm{\phi_1 - \phi_2}_\infty
\end{aligned}
\end{equation}

\section{Equivalent Problem Formulation} \label{appendix:equivalent_formulation}
Using the notation in this paper, the problem defined in \cite{ribeiro2012optimal} is
\begin{equation}
\label{eq:problem_translation_1}
\begin{aligned}
P^\star = &\max_{\boldsymbol{x}, \phi} f_0 (\boldsymbol{x})\\
&\text{s.t. } \boldsymbol{x} \leq \mathbb{E} \lbrack \boldsymbol{f_1}(s, \phi(s)) \rbrack \\
&\quad \boldsymbol{f_2}(\boldsymbol{x}) \geq 0, \boldsymbol{x} \in \mathcal{X}, \phi \in \Phi
\end{aligned}
\end{equation}
where $f_0$, and $\boldsymbol{f_2}$ are concave functions. $\boldsymbol{f_1}$ is not necessarily convex with respect to $\phi$. $\mathcal{X}$ is a convex set and $\Phi$ is compact. Note that $f_0$, $\boldsymbol{f_1}$, $\boldsymbol{f_2}$, $\boldsymbol{x}$, and $\mathcal{X}$ are not directly present in the \textit{Sufficiently Accurate} problem.

To translate the problem, let us assume that $x$ is $K+1$ dimensional, where $K$ is the number of constraints in \eqref{eq:functional_primal}. Let the last element of $\boldsymbol{f_1}(s, \phi)$ be equal to $-l(s, \phi) \mathbb{I}_0(s)$, the negative of the objective function in the \textit{Sufficiently Accurate} problem. Let the first $k^{\text{th}}$ element of $\boldsymbol{f_1}(s, \phi)$ be equal to $-g_k(s, \phi) \mathbb{I}_k(s)$, the negative of a constraint function in \eqref{eq:functional_primal}. Set the objective function to be $f_0(\boldsymbol{x}) = \boldsymbol{x}_{K+1}$, where $\boldsymbol{x}_{K+1}$ is the $(K+1)^{\text{th}}$ element of $\boldsymbol{x}$. $\boldsymbol{f_2}$ can be ignored by setting it to be the zero function. Under these assumptions, \eqref{eq:problem_translation_1} is equivalent to the following

\begin{equation}
\label{eq:problem_translation_2}
\begin{aligned}
    P^\star = &\max_{\boldsymbol{x}, \phi} \boldsymbol{x}_{K+1}\\
    &\text{s.t.  }  \boldsymbol{x}_k \leq -\mathbb{E} \lbrack g_k(s, \phi) \mathbb{I}_k(s) \rbrack, k=1, \hdots, K \\
    &\phantom{s.t.  } \boldsymbol{x}_{K+1} \leq -\mathbb{E}\lbrack l(s, \phi) \mathbb{I}_0(s) \rbrack \\
    &\phantom{s.t.  } \boldsymbol{x} \in \mathcal{X}, \phi \in \Phi.
\end{aligned}
\end{equation}
Now, define the set $\mathcal{X} = \{(x_1, x_2, \hdots, x_{K+1}) : x_k=0, k = 1, \hdots, K, x_{K+1} \in \mathcal{X}_{K+1} \}$, where $\mathcal{X}_{K+1}$ is an arbitrary compact set in one dimension. This set of vectors, $\mathcal{X}$, is a set that is 0 in the first $K$ components, and is compact in the last component. This, will further simplify \eqref{eq:problem_translation_2} to the following
\begin{equation}
\label{eq:problem_translation_3}
\begin{aligned}
    P^\star = &\max_{\phi}  -\mathbb{E}\lbrack l(s, \phi) \mathbb{I}_0(s) \rbrack\\
    &\text{s.t.  }  0 \leq -\mathbb{E} \lbrack g_k(s, \phi) \mathbb{I}_k(s) \rbrack, k=1, \hdots, K \\
    &\phantom{s.t.  } \phi \in \Phi
\end{aligned}
\end{equation}
as long as $l$ takes on values in a compact set. Finally, flipping all the negatives in \eqref{eq:problem_translation_3},
\begin{equation}
\label{eq:problem_translation_4}
\begin{aligned}
    P^\star = - &\min_{\phi \in \Phi} \mathbb{E}\lbrack l(s, \phi) \mathbb{I}_0(s) \rbrack\\
    &\text{s.t.  }  \mathbb{E} \lbrack g_k(s, \phi) \mathbb{I}_k(s) \rbrack \leq 0, k=1, \hdots, K \\
\end{aligned}
\end{equation}
This completes the translation of problem \eqref{eq:problem_translation_1} to \eqref{eq:functional_primal}.

\section{Proof of Theorem \ref{thm:eisen_parameterization}} \label{appendix:eisen_proof}
This proof follows some of the steps of the proof for \cite[Theorem 1]{eisen2018learning}.
Let $(\phi^\star, \lambda^\star)$ be the primal and dual variables that attains the solution value of $P^\star = D^\star$ in problems \eqref{eq:functional_primal} and \eqref{eq:functional_dual}. Similarly, let $(\theta^\star, \lambda_\theta^\star)$ be the primal and dual variables that attain the solution value $D_\theta^\star$ in problem \eqref{eq:parameterized_dual}. $\phi_{\theta^\star}$ is the function that $\theta^\star$ induces. Note that the optimal dual variables for \eqref{eq:functional_dual}, $\lambda^\star$, are not necessarily the same as the optimal dual variables for \eqref{eq:parameterized_dual}, $\lambda_\theta^\star$.

\subsection{Lower Bound}
We first show the lower bound for $D_\theta^\star$.
Writing out the dual problem \eqref{eq:functional_dual}, we obtain
\begin{equation}
    D^\star = \max_{\lambda \geq 0} \min_{\phi \in \Phi} \mathcal{L}(\phi, \lambda) = \mathcal{L}(\phi^\star, \lambda^\star)
\end{equation}
Since $\lambda^\star$ is the optimal dual variable that achieves the maximal value for the maximization and minimization for the Lagrangian, it is true that
\begin{equation}
    \phi^\star = \argmin_{\phi} \mathcal{L}(\phi, \lambda^\star).
\end{equation}
Thus for any $\phi$,
\begin{equation}
\label{eq:param_thm_1}
    \mathcal{L}(\phi^\star, \lambda^\star) \leq \mathcal{L}(\phi, \lambda^\star), \forall \phi \in \Phi.
\end{equation}
We now look at the parameterized dual problem \eqref{eq:parameterized_dual}.
\begin{equation}
    D_\theta^\star = \max_{\lambda \geq 0} \min_{\theta \in \Theta} \mathcal{L}_\theta(\theta, \lambda)
    = \max_{\lambda \geq 0} \min_{\theta \in \Theta} \mathcal{L}(\phi_\theta, \lambda)
\label{eq:param_thm_1_5}
\end{equation}
This simply redefines $\mathcal{L}_\theta$ in terms of $\mathcal{L}$ as the only difference is that $\mathcal{L}_\theta$ is only defined for a subset of the primal variables that $\mathcal{L}$ is defined for.
By definition, $\lambda_\theta^\star$ maximizes the minimization of $\mathcal{L}_\theta$ over $\theta$. That is to say for the dual solution $(\theta^\star, \lambda_\theta^\star)$, $\lambda_\theta^\star$ minimizes $\mathcal{L}(\phi_{\theta^\star}, \cdot)$
\begin{equation}
\begin{aligned}
    \lambda_\theta^\star &= \argmax_{\lambda \geq 0} \min_{\theta \in \Theta} \mathcal{L}(\phi_\theta, \lambda)\\
    &= \argmax_{\lambda \geq 0} \mathcal{L}(\phi_{\theta^\star}, \lambda).
\end{aligned}
\end{equation}
Thus, for all $\lambda$, it is the case that
\begin{equation}
\label{eq:param_thm_2}
    \mathcal{L}(\phi_{\theta^\star}, \lambda_\theta^\star)  \geq \mathcal{L}(\phi_{\theta^\star}, \lambda)
\end{equation}
Putting together \eqref{eq:param_thm_1} and \eqref{eq:param_thm_2}, we obtain
\begin{equation}
\begin{aligned}
    D_\theta^\star = \mathcal{L}(\phi_{\theta^\star}, \lambda_\theta^\star) &\geq \mathcal{L}(\phi_{\theta^\star}, \lambda^\star) \geq \mathcal{L}(\phi^\star, \lambda^\star) = D^\star
\end{aligned}
\end{equation}

\subsection{Upper Bound}
Next, we show the upper bound for $D_\theta^\star$. We begin by writing the Lagrangian \eqref{eq:parameterized_lagrangian}
\begin{equation}
    D_\theta^\star = \max_{\lambda \geq 0} \min_{\phi \in \Phi_\theta } \mathcal{L}(\phi_{\theta}, \lambda)
\end{equation}
as previously written in \eqref{eq:param_thm_1_5}. By adding and subtracting $\mathcal{L}(\phi, \lambda)$, we obtain
\begin{equation}
\label{eq:param_thm_3}
\begin{aligned}
    D_\theta^\star &= \max_{\lambda \geq 0} \min_{\theta \in \Theta } \mathcal{L}(\phi, \lambda) + \mathcal{L}(\phi_{\theta}, \lambda) - \mathcal{L}(\phi, \lambda)\\
    &= \max_{\lambda \geq 0} \lbrack \mathcal{L}(\phi, \lambda) + \min_{\theta \in \Theta } \lbrack \mathcal{L}(\phi_{\theta}, \lambda) - \mathcal{L}(\phi, \lambda) \rbrack \rbrack \\
    &\leq \max_{\lambda \geq 0} \lbrack \mathcal{L}(\phi, \lambda) + \min_{\theta \in \Theta } \abs{ \mathcal{L}(\phi_{\theta}, \lambda) - \mathcal{L}(\phi, \lambda) } \rbrack
\end{aligned}
\end{equation}
where the last line comes from the fact that the absolute value of an expression is always at least as large as the original expression, i.e. $x \leq \abs{x}$.
Looking just at the quantity $\abs{\mathcal{L}(\phi_\theta, \lambda) - \mathcal{L}(\phi, \lambda)}$,
we can expand it as
\begin{equation}
\begin{aligned}
    \abs{\mathcal{L}(\phi_\theta, \lambda) - \mathcal{L}(\phi, \lambda)} &=
    | \mathbb{E}\lbrack \ell_0(s, \phi_\theta) + {\lambda}^\top \boldsymbol{g}(s, \phi_\theta)\rbrack - \\
    &\mathbb{E}\lbrack \ell_0(s, \phi) + {\lambda}^\top \boldsymbol{g}(s, \phi)\rbrack |
\end{aligned}
\end{equation}
Using the triangle inequality, this is upper bounded as
\begin{equation}
\begin{aligned}
    \abs{\mathcal{L}(\phi_\theta, \lambda) - \mathcal{L}(\phi, \lambda)} &\leq
    \abs{\mathbb{E}\lbrack \ell_0(s, \phi_\theta) - \ell_0(s, \phi) \rbrack} +\\
    &\abs{\mathbb{E}\lbrack {\lambda}^\top (\boldsymbol{g}(s, \phi_\theta) - \boldsymbol{g}(s, \phi))\rbrack }
\end{aligned}
\end{equation}
Using H{\"o}lder's inequality, we can create a further upper bound
\begin{equation}
\begin{aligned}
    \abs{\mathcal{L}(\phi_\theta, \lambda) - \mathcal{L}(\phi, \lambda)} &\leq
    \norm{\mathbb{E}\lbrack \ell_0(s, \phi_\theta) - \ell_0(s, \phi) \rbrack}_\infty +\\
    & \norm{\lambda}_1 \norm{\mathbb{E}\lbrack \boldsymbol{g}(s, \phi_\theta) - \boldsymbol{g}(s, \phi)\rbrack }_\infty
\end{aligned}
\end{equation}
where the infinity norm of the scalar value $\mathbb{E}\lbrack \ell_0(s, \phi_\theta) - \ell_0(s, \phi) \rbrack$ is the same as its absolute value.
Using the fact that the infinity norm is convex and Jensen's inequality, we can move the norm inside of the expectation.
\begin{equation}
\begin{aligned}
    \abs{\mathcal{L}(\phi_\theta, \lambda) - \mathcal{L}(\phi, \lambda)} &\leq
    \mathbb{E}\norm{ \ell_0(s, \phi_\theta) - \ell_0(s, \phi)}_\infty +\\
    & \norm{\lambda}_1 \mathbb{E}\norm{ \boldsymbol{g}(s, \phi_\theta) - \boldsymbol{g}(s, \phi)}_\infty
\end{aligned}
\end{equation}
By expectation-wise Lipschitz-continuity of both the loss and constraint functions,
\begin{equation}
\label{eq:param_thm_4}
\begin{aligned}
    \abs{\mathcal{L}(\phi_\theta, \lambda) - \mathcal{L}(\phi, \lambda)} &\leq
    L \mathbb{E}\norm{ \phi_\theta - \phi}_\infty +\\
    & \quad \norm{\lambda}_1 L \mathbb{E}\norm{ \phi_\theta - \phi}_\infty \\
    &= (1 + \norm{\lambda}_1) L \mathbb{E} \norm{ \phi_\theta - \phi}_\infty.
\end{aligned}
\end{equation}
Combining \eqref{eq:param_thm_3} with \eqref{eq:param_thm_4}, we obtain
\begin{equation}
\label{eq:param_thm_5}
\begin{aligned}
    D_\theta^\star &\leq &= \max_{\lambda \geq 0} \lbrack \mathcal{L}(\phi, \lambda) + (\norm{\lambda}_1 + 1) L \min_{\theta \in \Theta}  \mathbb{E}\norm{ \phi_\theta - \phi}_\infty \rbrack 
\end{aligned}
\end{equation}
Since, $\Phi_\theta$ is an $\epsilon$-universal approximation for $\Phi$, we can write $\min_{\theta \in \Theta}  \mathbb{E}\norm{ \phi_\theta - \phi}_\infty \leq \epsilon$. This further reduces \eqref{eq:param_thm_5} to
\begin{equation}
\label{eq:param_thm_6}
    D_\theta^\star \leq \max_{\lambda \geq 0} \lbrack \mathcal{L}(\phi, \lambda) + (\norm{\lambda}_1 + 1)  L \epsilon \rbrack.
\end{equation}
Note that \eqref{eq:param_thm_6} is true for all $\phi$. In particular it must be also true for the $\lambda$ that minimizes the inner value, i.e.
\begin{equation}
\label{eq:param_thm_7}
\begin{aligned}
D_\theta^\star &\leq \max_{\lambda \geq 0} \min_{\phi \in \Phi} \lbrack \mathcal{L}(\phi, \lambda) + (\norm{\lambda}_1 + 1)  L \epsilon \rbrack\\
&= L \epsilon + \max_{\lambda \geq 0} \min_{\phi \in \Phi} \lbrack \mathbb{E}\lbrack \ell_0(s, \phi)\rbrack + \lambda^\top(\mathbb{E}\lbrack \boldsymbol{g}(s, \phi)\rbrack + \boldsymbol{1} L \epsilon) \rbrack
\end{aligned}
\end{equation}
The second half of \eqref{eq:param_thm_7} is actually the solution to the dual problem \eqref{eq:perturbed_dual}. The primal problem is reproduced here for reference,
\begin{equation*}
\begin{aligned}
P_{L\epsilon}^\star = &\min_{\phi \in \Phi} \mathbb{E}\lbrack \ell_0(s, \phi)\rbrack \\
&\text{s.t.  } \mathbb{E}\lbrack \boldsymbol{g}(s, \phi)\rbrack + \boldsymbol{1} L \epsilon \leq 0.
\end{aligned}
\end{equation*}
That is to say, $D_\theta^\star \leq L \epsilon + D_{L\epsilon}^\star$.
The primal problem \eqref{eq:perturbed_primal} is a perturbed version of \eqref{eq:functional_primal}, where all the constraints are tighter by $L \epsilon$. There exists a relationship between the solution of \eqref{eq:perturbed_primal} and \eqref{eq:functional_primal} from \cite[Eq. 5.57]{boyd2004convex}. Treating \eqref{eq:functional_primal} as the perturbed version of \eqref{eq:perturbed_primal} (that tightens the constraints by $-L\epsilon$), the relationship between the two solutions is
\begin{equation}
    P^\star \geq P_{L\epsilon}^\star - {\lambda_{L\epsilon}^\star}^\top \boldsymbol{1} L \epsilon.
\end{equation}
Since both \eqref{eq:functional_primal} and \eqref{eq:perturbed_primal} have zero duality gap by Theorem \ref{thm:ribeiro_zero_duality}, this is the same as
\begin{equation}
\label{eq:param_thm_8}
    D^\star \geq D_{L\epsilon}^\star - \norm{\lambda_{L\epsilon}^\star}_1 L \epsilon.
\end{equation}
Combining \eqref{eq:param_thm_8} with the fact that $D_\theta^\star \leq L\epsilon + D_{L\epsilon}^\star$, the following bound is obtained.
\begin{equation}
    D_\theta^\star \leq D^\star + L\epsilon (\norm{\lambda_{L\epsilon}^\star}_1 + 1)
\end{equation}
This gives us the desired upper bound.

\section{Proof of Lemma \ref{lemma:empirical_error}}
\label{appendix:lemma_proof}

We start by establishing an upper bound on the difference $\abs{D_\theta^\star-D_N^\star}$. By definition of the Dual Problems \eqref{eq:parameterized_dual} and \eqref{eq:empirical_dual}, it follows that 
 $D_\theta^\star  = \min_\theta \mathcal{L}_\theta(\theta, \lambda_\theta^\star)$ and $D_N^\star =  \min_\theta \mathcal{L}_N(\theta, \lambda_N^\star)$. Hence we have that 
 \begin{align}
	D_\theta^\star - D_N^{\star}  &= \min_\theta \mathcal{L}_\theta(\theta, \lambda_\theta^\star) - \min_\theta \mathcal{L}_N(\theta, \lambda_N^\star) \nonumber \\
	& \leq \min_\theta \mathcal{L}_\theta(\theta, \lambda_\theta^\star) - \min_\theta \mathcal{L}_N(\theta, \lambda_\theta^\star), \label{eq:upper2} 
\end{align}

 where the inequality follows from the fact that $\lambda_N^\star$ maximizes the function $d_N(\lambda) = \min_\theta\mathcal{L}_N(\theta, \lambda)$. Thus, any other $\lambda$, in particular $\lambda_\theta^\star$ results in a value that is less than or equal to $D_N^\star$. Let  $\theta_N^\star(\lambda) =\argmin_{\theta\in\Theta} \mathcal{L}_N(\theta,\lambda) $. Substituting by this definition and using the definition of minimum, \eqref{eq:upper2} can be further upper bounded by 
\begin{equation}\label{eq:upper_bound}
	D_\theta^\star - D_N^{\star}   \leq \mathcal{L}_\theta(\theta_N^\star(\lambda_\theta^\star), \lambda_\theta^\star) - \mathcal{L}_N(\theta_N^\star(\lambda_\theta^\star), \lambda_\theta^\star). 
\end{equation}
We set now to establish a similar lower bound. Analogous to the step for the upper bound, we can use the definition of the Dual problem to lower bound $D_\theta^\star - D_N^\star $

\begin{equation}
	D_\theta^\star - D_N^\star \geq \min_\theta \mathcal{L}_\theta(\theta, \lambda_N^\star) - \min_\theta \mathcal{L}_N(\theta, \lambda_N^\star) \label{eq:lower2}. 
\end{equation}
Likewise, define $\theta^\star(\lambda) =\argmin_{\theta\in\Theta} \mathcal{L}_\theta(\theta,\lambda) $ and further upper bound \eqref{eq:lower2} as
\begin{equation}
	D_\theta^\star - D_N^\star \geq \mathcal{L}_\theta(\theta^\star(\lambda_N^\star), \lambda_N^\star) - \mathcal{L}_N(\theta^\star(\lambda_N^\star), \lambda_N^\star). \label{eq:lower_bound}
\end{equation}

\balance

Using the upper and lower bounds for $D_\theta^\star-D_N^\star$ derived in \eqref{eq:upper_bound} and \eqref{eq:lower_bound}, we can bound the absolute value of the difference by the maximum of the absolute values of the lower and upper bounds
\begin{equation}
\label{eq:bound_max}
|D_\theta^\star - D_N^{\star} | \leq 
\max
\begin{cases}
|\mathcal{L}_\theta(\theta^\star(\lambda_N^\star), \lambda_N^\star) - \mathcal{L}_N(\theta^\star(\lambda_N^\star), \lambda_N^\star) |\\
|\mathcal{L}_\theta(\theta_N^\star(\lambda_\theta^\star), \lambda_\theta^\star) - \mathcal{L}_N(\theta_N^\star(\lambda_\theta^\star), \lambda_\theta^\star) |
\end{cases}
\end{equation}
A conservative upper bound for the previous expression is 
\begin{equation}
|D_\theta^\star - D_N^{\star} | \leq \abs{\sup_{\theta\in\Theta,\lambda \in \mathbb{R}_+^K} \mathcal{L}_\theta(\theta, \lambda) - \mathcal{L}_N(\theta, \lambda) }
\end{equation}
This completes the proof of the Lemma.

% \textcolor{red}{
% TODO: replace all begin{align} with begin{equation}begin{aligned}
% TODO: replace all norms with $\norm{x}$ or $\abs{x}$
% TODO: put brackets in each expectations using $\lbrack \rbrack$
% TODO: change all transposes into $a^\top$
% TODO: make sure there are no blambda or maxlambda
% TODO: get rid of the word "like"
% TODO: update all tables and figures and their captions
% TODO: make sure to use Figure instead of Fig.
% TODO: replace all optimal variables with $x^\star$ instead of $x^*$
% TODO: make sure L in lagrangian function is using mathcal
% }

\bibliographystyle{IEEEtran}
\bibliography{IEEEabrv,bib.bib}

\end{document}